\documentclass[sigconf]{acmart}

\usepackage{subcaption}
\usepackage{booktabs}
\usepackage{multirow}
\usepackage{amsmath}

\usepackage{amssymb}
\usepackage{mathtools}
\usepackage{amsthm}
\usepackage{diagbox}

\theoremstyle{plain}
\newtheorem{theorem}{Theorem}[section]
\newtheorem{proposition}[theorem]{Proposition}

\theoremstyle{definition}

\newtheorem{assumption}[theorem]{Assumption}
\theoremstyle{remark}
\newtheorem{remark}[theorem]{Remark}
\AtBeginDocument{%
  }

\setcopyright{none}

\renewcommand\footnotetextcopyrightpermission[1]{}

\begin{document}

\title[Beyond First-Order: Learning Riemannian Geometries for Invariant Visual Place Recognition]{Beyond First-Order: Learning Riemannian Geometries for Invariant Visual Place Recognition}

\author{Jintao Cheng\textsuperscript{1}, Weibin Li\textsuperscript{2}, Jiazheng Huang\textsuperscript{3}, Mingkai Jia\textsuperscript{1}, Xieyuanli Chen\textsuperscript{4}, Zhijian He\textsuperscript{5}, Jin Wu\textsuperscript{6}, Chi Man Vong\textsuperscript{2} and Wei Zhang\textsuperscript{1}}
\affiliation{%
  \institution{\textsuperscript{1}\textit{The Hong Kong University of Science and Technology, Hong Kong, China}\\
  \textsuperscript{2}\textit{University of Macau, Macau, China}\\
  \textsuperscript{3}\textit{South China Normal University, Guangzhou, China}\\
  \textsuperscript{4}\textit{National University of Defense Technology, Changsha, China}\\
  \textsuperscript{5}\textit{Shenzhen Technology University, Shenzhen, China}\\
  \textsuperscript{6}\textit{University of Science and Technology Beijing, Beijing, China}}
  \country{}}

\renewcommand{\shortauthors}{Cheng et al.}
\begin{abstract}
Visual Place Recognition (VPR) demands representations robust to drastic environmental and viewpoint shifts.  Existing aggregation paradigms either depend on extensive supervised training or rely on first-order pooling, often struggling to preserve structural correlations under extreme shifts or incurring high adaptation costs.  In this work, we propose Riemannian Invariant Aggregation (RIA), a unified geometric framework that explicitly models second-order scene structure on the Symmetric Positive Definite (SPD) manifold. By treating perturbations as tractable congruence transformations, RIA leverages geometry-aware Riemannian mappings to project covariance descriptors into a linearized Euclidean space, effectively preserving invariant structural components while suppressing noise. Extensive evaluations demonstrate that RIA achieves zero-shot performance comparable to supervised methods, and establishes state-of-the-art accuracy with simple fine-tuning, particularly in unstructured environments. The source code will be released.
\end{abstract}


\maketitle
\begingroup
\renewcommand\thefootnote{}
\footnotetext{J. Cheng, Z. He, and W. Zhang are with The Hong Kong University of Science and Technology (e-mail: Jchengau@connect.ust.hk; eeweiz@ust.hk).
Corresponding author: Wei Zhang.}
\endgroup

\section{Introduction}

Visual Place Recognition (VPR) constitutes a cornerstone of long-term autonomous navigation, enabling robots to accurately localize within previously visited environments despite drastic changes in visual appearance \cite{lowry2015visual, zhang2021visual}. The fundamental challenge of this task stems from the stringent requirement to simultaneously satisfy two distinct and often conflicting objectives: \textbf{condition invariance} \cite{milford2012seqslam, sattler2018benchmarking} and \textbf{viewpoint invariance} \cite{hausler2021patch, berton2022rethinking}, as illustrated in Fig. \ref{fig:1}. On one hand, environmental dynamics---such as day-to-night transitions, seasonal cycles, and weather fluctuations---induce drastic non-linear photometric distortions that severely degrade feature stability \cite{zaffar2021vpr}. On the other hand, variations in the camera's 6-DoF pose result in complex projective transformations, scale changes, and occlusions, thereby disrupting spatial correspondence \cite{warburg2020mapillary}. Consequently, constructing a robust representation capable of withstanding these compounded perturbations while capturing the intrinsic essence of a scene remains a core scientific problem in the field.

\begin{figure}[ht]
  \centering
  \includegraphics[width=0.47\textwidth]{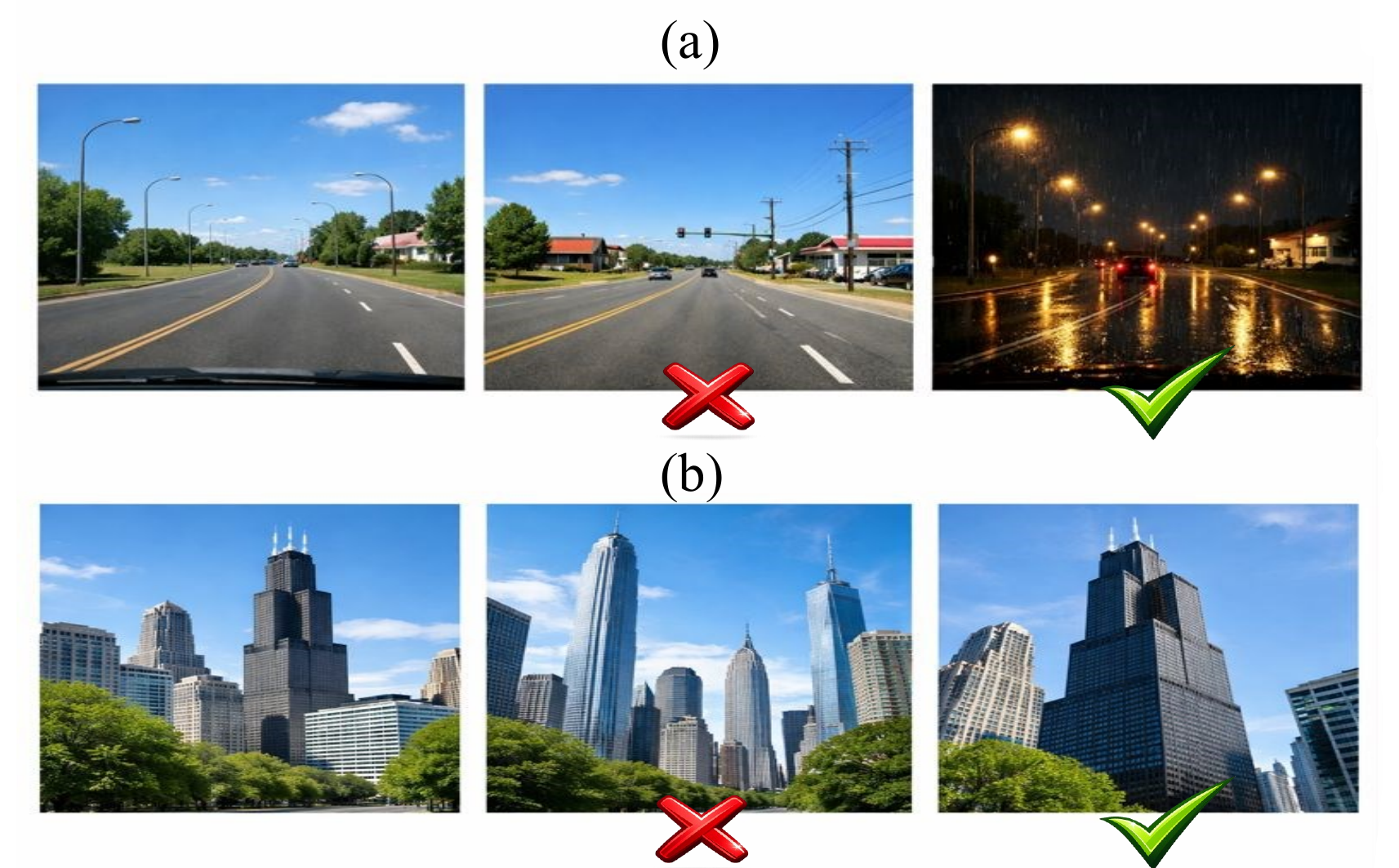}
  \caption{Illustration of the two fundamental challenges in Visual Place Recognition (VPR). (a) Condition Invariance: The system must identify the correct match (right) despite drastic illumination and weather changes (e.g., sunny day vs. rainy night), while rejecting perceptually similar scenes from different locations (middle). (b) Viewpoint Invariance: The system must recognize the same landmark (right) under significant changes in scale and camera pose, distinguishing it from other similar-looking structures (middle).}
  \label{fig:1}
\end{figure}

\par
To obtain such invariant representations, modern VPR methodologies have converged to a two-stage paradigm: a feature extraction Backbone $\Phi$ and a global Aggregation Head $\Psi$. Specifically, the backbone $\Phi$ is responsible for extracting discriminative local features from raw images, while the aggregation head $\Psi$ fuses these local cues into a compact global descriptor. Historically, extensive research efforts have been dedicated to training specialized backbones \cite{arandjelovic2016netvlad, berton2022rethinking, zhu2023r2former}. However, with the emergence of powerful visual foundation models represented by the DINO series \cite{zhang2022dino}, the capability to extract robust local features has become increasingly mature and generalized. The critical question now lies in how to organize these high-quality features to construct a global descriptor that remains stable under the aforementioned geometric and photometric transformations.
\par
Existing aggregation methods primarily fall into two categories, yet both face distinct dilemmas. The first category is Supervised Implicit Aggregation \cite{ali2023mixvpr, lu2024cricavpr}. These methods employ massive Multi-Layer Perceptrons (MLPs) or attention mechanisms, attempting to implicitly approximate the compositional rules between features through intensive training on specific datasets. However, this strategy essentially encodes domain-specific data-driven priors rather than learning general physical laws, leading to heavy reliance on training data and limited generalization capabilities in cross-domain scenarios. The second category is Unsupervised First-order Aggregation \cite{keetha2023anyloc, malone2025hyperdimensional}. To maintain an unsupervised nature, these methods often revert to rudimentary aggregation schemes, such as Generalized Mean (GeM) pooling \cite{radenovic2018fine} or VLAD clustering. While these methods based on First-order Statistics are general, they neglect the rich structural correlations among features, resulting in inherent sensitivity to environmental perturbations.
\par
We advocate revisiting this problem from a geometric perspective to overcome these limitations. Drastic environmental changes (such as illumination scaling and viewpoint rotation) can be approximately modeled as affine transformations within the feature space. We demonstrate through theoretical derivation and extensive experiments that, first-order statistics (mean) are unstable and drift with transformations. However, second-order statistics (covariance) exhibit a unique advantage---they possess an intrinsic Congruence Property. This implies that the covariance matrix can naturally resist additive noise and maintain the consistency of the topological structure of feature distribution under multiplicative transformations. Mathematically, this second-order structure resides on the Symmetric Positive Definite (SPD) manifold \cite{huang2017riemannian}. This provides us with a theoretical foundation to achieve robustness relying solely on the intrinsic geometric properties of the data, without depending on parameter learning.
\par
In this work, we introduce Riemannian Invariant Aggregation (RIA) operator, an explicit geometric modeling paradigm. RIA explicitly extracts the intrinsic second-order covariance structure of the scene, obviating the need for learned parameters to approximate structural information. We leverage Riemannian Manifold Mapping to project the non-Euclidean geometric structure into a statistically well-behaved tangent space, thereby achieving compatibility with standard Euclidean retrieval. By leveraging analytical statistical properties and geometry-aware manifold mapping, RIA establishes a \textit{unified aggregation framework} that bridges zero-shot deployment and target-domain adaptation. This design preserves theoretical interpretability while delivering competitive zero-shot robustness across diverse environmental shifts, and further exploits these geometric priors to achieve state-of-the-art accuracy through fine-tuning.
\par

The principal contributions of this work are summarized as follows:
\par
1) We provide a theoretical analysis from a Riemannian perspective, showing that second-order statistics residing on the SPD manifold exhibit enhanced stability under viewpoint and illumination variations compared to conventional first-order aggregates.
\par
2) We introduce the Riemannian Invariant Aggregation (RIA) operator, a \textit{unified geometric framework} that seamlessly supports both \textbf{zero-shot deployment} and \textbf{fine-tuning}. By projecting second-order features onto the tangent space, RIA produces compact descriptors that accelerate retrieval and minimize memory overhead without accuracy loss.
\par
3) Extensive evaluations across diverse VPR benchmarks, encompassing a wide range of environmental conditions, demonstrate that RIA achieves zero-shot performance comparable to supervised methods. Furthermore, RIA establishes state-of-the-art results with fine-tuning, particularly in challenging unstructured environments. The source code will be released as open source.

\section{Related Work}

\subsection{Visual Place Recognition}
Contemporary VPR pipelines generally adhere to a two-stage paradigm: extracting local features via a backbone and fusing them into a global descriptor. Early approaches primarily fine-tuned CNNs \cite{arandjelovic2016netvlad, radenovic2018fine}, later evolving into massive classification frameworks to learn viewpoint robustness \cite{berton2022rethinking, berton2023eigenplaces}. With the advent of Vision Transformers, recent methods integrate attention mechanisms to capture long-range dependencies \cite{wang2022transvpr, zhu2023r2former}. While foundation models have demonstrated strong zero-shot capabilities \cite{keetha2023anyloc}, adapter-based approaches often re-introduce training phases to align features with the target domain \cite{lu2024towards, lu2024cricavpr}.
\par
The aggregation module plays a pivotal role in encoding invariance. Traditional methods rely on first-order statistics or clustering mechanisms \cite{arandjelovic2016netvlad, radenovic2018fine, lu2024supervlad}. To capture complex structural relationships, supervised methods employ parameter-heavy MLPs or cross-image attention \cite{ali2023mixvpr, lu2024cricavpr}, though they often suffer from limited generalization across domains. Conversely, unsupervised methods typically revert to rudimentary pooling or optimal transport schemes \cite{keetha2023anyloc, izquierdo2024optimal} to maintain their training-free nature.

\subsection{Deep Learning on SPD Manifolds}

The integration of Riemannian geometry into deep neural networks has significantly advanced the processing of second-order statistics. This paradigm was pioneered by SPDNet\cite{huang2017riemannian}, which generalized conventional CNN operations---such as bilinear mapping and eigenvalue rectification---to the SPD manifold. Building on this, DreamNet\cite{wang2022dreamnet} bridged Euclidean and Riemannian domains to learn more discriminative deep representations. Subsequent research has focused on enhancing both representational power and training efficiency. Notable contributions include introducing Riemannian local mechanisms to capture fine-grained structure\cite{chen20231}, meta-learning optimizers to handle manifold constraints\cite{gao2020learning}, and exploring space quantization for efficient representation\cite{tang2020generalized}. Besides, these architectural advancements have driven progress in metric and similarity learning, particularly for Image Set Classification. Huang et al.\cite{huang2017geometry} proposed geometry-aware similarity learning to handle visual distortions, while Wang et al.\cite{wang2022deep} developed deep metric learning frameworks specifically tailored for set-based classification on the manifold.

However, the potential of SPD manifold learning in Visual Place Recognition (VPR) remains largely untapped. Unlike Image Set Classification, which typically operates in closed-set scenarios, VPR demands descriptors that are explicitly robust to drastic, open-world environmental shifts and viewpoint deviations. Our work addresses this gap by leveraging the intrinsic geometry of SPD matrices to construct a second-order representation that ensures geometric stability in challenging open-world scenarios.

\section{Preliminaries}
\label{sec:preliminaries}

In this section, we formulate the geometric foundation of modeling image representations on the Symmetric Positive Definite (SPD) manifold. Throughout the paper, vectors are denoted by bold lower-case letters (e.g., $\boldsymbol{x}$), and matrices by bold upper-case letters (e.g., $\boldsymbol{C}$).

\textbf{The SPD Manifold.} 
The space of $d \times d$ Symmetric Positive Definite matrices, denoted as $\mathcal{S}_{++}^d$, is defined as:
\begin{equation}
\begin{split}
    \mathcal{S}_{++}^d = \Big\{ & \boldsymbol{M} \in \mathbb{R}^{d \times d} \mid \boldsymbol{M} = \boldsymbol{M}^\top, \\
    & \boldsymbol{v}^\top \boldsymbol{M} \boldsymbol{v} > 0, \forall \boldsymbol{v} \in \mathbb{R}^d \setminus \{\mathbf{0}\} \Big\}.
\end{split}
\end{equation}
Geometrically, $\mathcal{S}_{++}^d$ forms a convex cone endowed with a Riemannian metric. Standard Euclidean operations on this manifold ignore the intrinsic geodesic curvature, often leading to the ``swelling effect,'' where the determinant of the matrix increases artificially, introducing noise into the representation.

\textbf{Power Euclidean Metric (PEM).} 
To perform efficient retrieval while respecting the manifold geometry, we adopt the Power Euclidean Metric\cite{dryden2010power} framework. This maps the Riemannian manifold to a linearized metric space via a matrix power transformation. Given any two descriptors $\boldsymbol{C}_1, \boldsymbol{C}_2 \in \mathcal{S}_{++}^d$, the PEM distance with power $\alpha \in (0, 1]$ is defined as:
\begin{equation}
    \label{eq:pem_dist}
    d_{\text{PEM}}(\boldsymbol{C}_1, \boldsymbol{C}_2) = \frac{1}{\alpha} \left\| \boldsymbol{C}_1^\alpha - \boldsymbol{C}_2^\alpha \right\|_F.
\end{equation}
In this work, we focus on the matrix square root ($\alpha = 0.5$), which effectively flattens the manifold curvature and approximates the Riemannian geodesic distance. To validate the effectiveness of the power parameter $\alpha$ and determine its optimal value for our task, we conduct extensive experiments on this parameter, with detailed analyses and results presented in the section \ref{app:Power Parameter}.

\textbf{Isometric Vectorization.} 
To interface with standard vector-based search engines, a symmetric matrix $\boldsymbol{M}$ must be flattened. To ensure this mapping $\psi: \mathcal{S}_{++}^d \to \mathbb{R}^{d(d+1)/2}$ is an isometry, the off-diagonal elements must be scaled:
\begin{equation}
    \label{eq:vectorization}
    \scalebox{0.9}{$\operatorname{vec}(\boldsymbol{M}) = \left[ m_{1,1}, \ldots, m_{d,d}, \sqrt{2}m_{1,2}, \ldots, \sqrt{2}m_{d-1,d} \right]^\top$}.
\end{equation}

\section{Methodology}
\label{sec:method}

\setlength{\textfloatsep}{8pt}

\begin{figure*}[ht]
  \centering
  \includegraphics[width=0.9\textwidth]{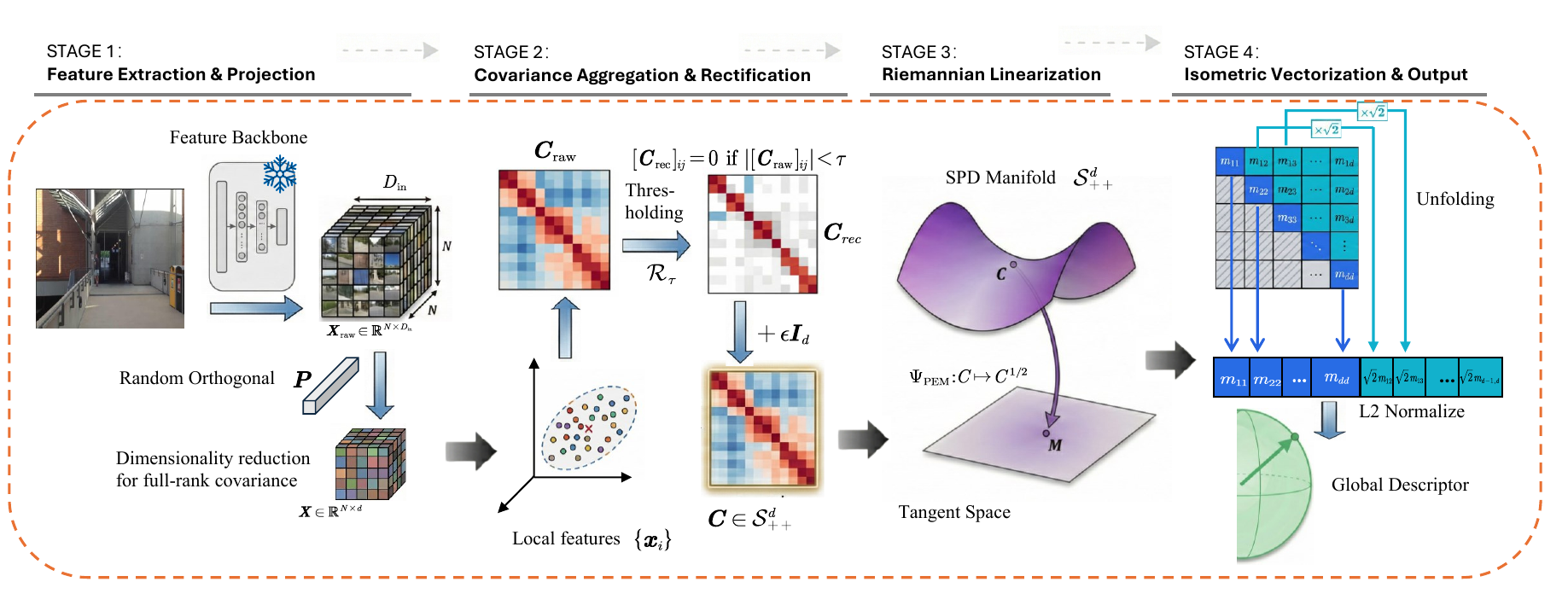}
  \caption{\textbf{Schematic overview of the proposed Riemannian Invariant Aggregation (RIA) framework.} The pipeline transforms local features from a frozen backbone into a robust global descriptor through four geometric phases: \textbf{Stage 1:} High-dimensional features are projected onto a lower-dimensional subspace to ensure a full-rank covariance estimation. \textbf{Stage 2:} We compute the sample covariance and apply ReCov to suppress spurious noise. \textbf{Stage 3:} The covariance descriptor on the SPD manifold is mapped to a linearized tangent space via the PEM, approximated by Newton-Schulz iterations. \textbf{Stage 4:} The matrix is flattened using isometric vectorization (scaling off-diagonals by $\sqrt{2}$) and $L_2$ normalized to produce the final retrieval-ready descriptor.}
  \label{fig:frame}
\end{figure*}

We introduce the Riemannian Invariant Aggregation (RIA) framework, a unified pipeline designed to construct robust scene descriptors for VPR, as shown in Fig. \ref{fig:frame}. Our method transforms local deep features into a global descriptor on the SPD manifold through four logically coupled phases: (1) \textbf{feature projection and aggregation} to construct the compact covariance representation, (2) \textbf{sparse structural rectification} to suppress spurious correlations, (3) \textbf{iterative Riemannian linearization} to approximate the geodesic metric efficiently, and (4) \textbf{isometric vectorization} to bridge the manifold geometry with Euclidean retrieval engines.

\subsection{Feature Projection and Aggregation}
\label{sec:feature_agg}
Current VPR approaches typically rely on first-order pooled features, which suffer from perceptual aliasing, are sensitive to viewpoint changes, and lack robustness under illumination variations. We theoretically establish that second-order statistics effectively mitigate these issues by encoding structural correlations and offering inherent robustness to geometric and photometric transformations; hence, we adopt the covariance matrix as a feature representation at this stage to enhance descriptor robustness.

Given an input image $I$, we utilize a pre-trained Visual Foundation Model to extract dense local representations. We extract a set of $N$ patch-level local features, denoted as a matrix $\boldsymbol{X}_{\text{raw}} \in \mathbb{R}^{N \times D_{\text{in}}}$, where $D_{\text{in}}$ represents the native feature dimension. To ensure the resulting covariance matrix is full-rank and strictly resides on the SPD manifold, we project the features into a lower-dimensional subspace $\mathbb{R}^d$ ($d < N$) using a fixed random orthogonal matrix $\boldsymbol{P} \in \mathbb{R}^{D_{\text{in}} \times d}$:
\begin{equation}
    \boldsymbol{X} = \boldsymbol{X}_{\text{raw}} \boldsymbol{P},
\end{equation}
where $\boldsymbol{X} = [\boldsymbol{x}_1, \dots, \boldsymbol{x}_N]^\top \in \mathbb{R}^{N \times d}$ represents the projected feature matrix. Here, $\boldsymbol{x}_i \in \mathbb{R}^d$ denotes the $i$-th projected local descriptor, which is the $i$-th row of $\boldsymbol{X}$ treated as a column vector. Subsequently, we aggregate these descriptors into a global sample covariance matrix $\boldsymbol{C}_{\text{raw}} \in \mathbb{R}^{d \times d}$:
\begin{equation}
    \boldsymbol{C}_{\text{raw}} = \frac{1}{N-1} \sum_{i=1}^{N} (\boldsymbol{x}_i - \bar{\boldsymbol{x}})(\boldsymbol{x}_i - \bar{\boldsymbol{x}})^\top,
\end{equation}
where $\bar{\boldsymbol{x}} = \frac{1}{N}\sum_{i=1}^N \boldsymbol{x}_i$ is the mean feature vector.

\subsection{Sparse Structural Rectification}
\label{sec:recov}

Sample covariance matrices derived from high-dimensional deep features often serve as noisy estimators of the underlying scene geometry. Specifically, when the feature dimension $d$ is large relative to the number of patches $N$, the off-diagonal entries of $\boldsymbol{C}_{\text{raw}}$ are prone to spurious correlations that distort the Riemannian structure. To recover the latent salient structure, we employ a Rectified Covariance(ReCov) mechanism, which can be viewed as a non-linear thresholding operator $\mathcal{R}_{\tau}(\cdot)$ aimed at consistent covariance estimation in high dimensions.

We define the rectified matrix $\boldsymbol{C}_{\text{rec}} = \mathcal{R}_{\tau}(\boldsymbol{C}_{\text{raw}})$ element-wise as:
\begin{equation}
    [\boldsymbol{C}_{\text{rec}}]_{ij} = 
    \begin{cases} 
    [\boldsymbol{C}_{\text{raw}}]_{ij} & \text{if } i=j \text{ or } |[\boldsymbol{C}_{\text{raw}}]_{ij}| > \tau, \\ 
    0 & \text{otherwise},
    \end{cases}
\end{equation}
where $\tau \ge 0$ is the structural saliency threshold. Hard thresholding enables consistent estimation of sparse covariance matrices by removing noise-induced correlations while preserving dominant structural components. Although it may compromise positive definiteness, the subsequent regularization step ensures geometric integrity on the SPD manifold.

To ensure the resulting matrix strictly resides in the interior of the SPD manifold and to maintain numerical stability during the subsequent matrix square root computation, we apply the following regularization to obtain the final SPD descriptor $\boldsymbol{C}$:
\begin{equation}
    \boldsymbol{C} = \boldsymbol{C}_{\text{rec}} + \epsilon \boldsymbol{I}_d,
\end{equation}
where $\boldsymbol{C} \in \mathcal{S}_{++}^d$ represents the regularized covariance matrix, $\boldsymbol{I}_d \in \mathbb{R}^{d \times d}$ denotes the identity matrix, and $\epsilon > 0$ is a small regularization constant. This step effectively shifts the eigenvalues away from zero, yielding a strictly positive definite representation that captures the stable second-order geometric signature of the scene.

\subsection{Iterative Riemannian Linearization}
\label{sec:newton_schulz}

A core challenge in manifold learning involves comparing descriptors using the intrinsic geodesic distance, which is computationally expensive due to its heavy reliance on matrix logarithms and inversions. To circumvent this, we adopt the PEM framework to map the curved Riemannian manifold to a linearized metric space where Euclidean operations remain geometry-aware. Formally, our objective is to compute the matrix square root $\boldsymbol{C}^{1/2}$ for the regularized covariance matrix $\boldsymbol{C} \in \mathcal{S}_{++}^d$ .

While $\boldsymbol{C}^{1/2}$ can be obtained via standard Eigenvalue Decomposition (EIG), this approach typically exhibits $O(d^3)$ complexity and gradient instability on GPU devices. Instead, we employ the coupled Newton-Schulz iteration~\cite{li2018towards} for an efficient numerical approximation. Since this iterative process converges locally only when the spectral radius is less than one, we first perform pre-normalization:
\begin{equation}
    \label{eq:pre_norm}
    \boldsymbol{A}_0 = \frac{1}{\|\boldsymbol{C}\|_F} \boldsymbol{C},
\end{equation}
where $\|\boldsymbol{C}\|_F = \sqrt{\text{tr}(\boldsymbol{C}^\top \boldsymbol{C})}$ denotes the Frobenius norm. This normalization ensures the eigenvalues of $\boldsymbol{A}_0$ reside within the convergence domain $(0, 1]$. We then apply the coupled update rules for $k = 1, \dots, K$. Initializing $\boldsymbol{Y}_0 = \boldsymbol{A}_0$ and $\boldsymbol{Z}_0 = \boldsymbol{I}_d$, the iteration proceeds as:
\begin{align}
    \boldsymbol{Y}_{k} &= \frac{1}{2}\boldsymbol{Y}_{k-1}(3\boldsymbol{I}_d - \boldsymbol{Z}_{k-1}\boldsymbol{Y}_{k-1}), \\
    \boldsymbol{Z}_{k} &= \frac{1}{2}(3\boldsymbol{I}_d - \boldsymbol{Z}_{k-1}\boldsymbol{Y}_{k-1})\boldsymbol{Z}_{k-1},
\end{align}
where $\boldsymbol{Y}_k$ and $\boldsymbol{Z}_k$ are the iterative approximations of the square root and the inverse square root of $\boldsymbol{A}_0$, respectively.

After $K$ iterations, the matrix $\boldsymbol{Y}_K$ provides a numerical approximation of $\boldsymbol{A}_0^{1/2}$. To recover the original physical scale of the representation, we perform a post-compensation step to obtain the final mapped descriptor $\boldsymbol{M}$:
\begin{equation}
    \label{eq:post_compensation}
    \boldsymbol{M} = \boldsymbol{Y}_K \sqrt{\|\boldsymbol{C}\|_F}.
\end{equation}

This compensation is essential because the pre-normalization in Eq.~\eqref{eq:pre_norm} rescales the spectral energy of the matrix; failing to counteract this scaling would destroy the magnitude information of the original covariance, which carries critical discriminative cues for place recognition.

\subsection{Isometric Vectorization}
\label{sec:vectorization}
The final phase of our framework involves projecting the linearized manifold representation into a Euclidean space suitable for high-speed indexing. To faithfully preserve the geometric properties of the SPD manifold during this transformation, we adopt the mapping $\psi: \text{Sym}(d) \to \mathbb{R}^D$ ($D = d(d+1)/2$) defined in Equation~\ref{eq:vectorization}. This mapping converts a symmetric matrix $\boldsymbol{M}$ into a global descriptor $\boldsymbol{v}$. 

The inclusion of the $\sqrt{2}$ scaling factor for off-diagonal entries is not a mere heuristic but a mathematical prerequisite to establish a Hilbert space isomorphism. This weighting ensures that the Euclidean inner product between any two vectorized descriptors $\boldsymbol{v}_a$ and $\boldsymbol{v}_b$ is strictly equivalent to the Frobenius inner product of their corresponding manifold points:
\begin{equation}
    \label{eq:inner_product_preservation}
    \langle \boldsymbol{v}_a, \boldsymbol{v}_b \rangle = \langle \boldsymbol{M}_a, \boldsymbol{M}_b \rangle_F = \text{tr}(\boldsymbol{M}_a \boldsymbol{M}_b).
\end{equation}

By establishing this isomorphism, we guarantee that the $L_2$ normalization of $\boldsymbol{v}$ correctly maps the descriptor onto a unit hypersphere while maintaining the structural correlations captured by the SPD manifold. Crucially, this normalization, when coupled with the PEM mapping, grants the descriptor intrinsic invariance to global intensity fluctuations.

\textbf{Geometric Integrity.} 
The significance of this isometric embedding is that it allows standard Euclidean retrieval algorithms to operate as if they were performing computations directly on the manifold tangent space. By ensuring that $\boldsymbol{v}_a^\top \boldsymbol{v}_b$ faithfully represents the second-order structural similarity, our framework maintains high discriminative power and geometric stability during large-scale place recognition, effectively bridging Riemannian theory with practical search efficiency.

\section{Experiments}

\begin{table*}[htbp]
    \centering
    \setlength{\tabcolsep}{8pt}
    \caption{\textbf{Comparison across multiple VPR benchmarks.} }
    \label{tab:main_results}
    
    \resizebox{\textwidth}{!}{%
        \begin{tabular}{l c c c c c c c c c c c}
            \toprule
            \multirow{2.5}{*}{\textbf{Method}} & \multirow{2.5}{*}{\textbf{Desc. Dim.}} & \multicolumn{2}{c}{\textbf{17places}} & \multicolumn{2}{c}{\textbf{Pitts30k}} & \multicolumn{2}{c}{\textbf{Gardens}} & \multicolumn{2}{c}{\textbf{Oxford}} & \multicolumn{2}{c}{\textbf{St. Lucia}} \\
            \cmidrule(lr){3-4} \cmidrule(lr){5-6} \cmidrule(lr){7-8} \cmidrule(lr){9-10} \cmidrule(lr){11-12}
             & & R@1 & R@5 & R@1 & R@5 & R@1 & R@5 & R@1 & R@5 & R@1 & R@5 \\
            \midrule
            
            \multicolumn{12}{l}{\textit{Supervised Learning}} \\
            NetVLAD~\cite{arandjelovic2016netvlad} & 32,768 & 61.6 & 77.8 & 86.1 & 92.7 & 58.5 & 85.0 & 57.6 & 79.1 & 57.9 & 73.0 \\
            CosPlace~\cite{berton2022rethinking} & 512 & 61.1 & 76.1 & 90.4 & 95.7 & 74.0 & 94.5 & 95.3 & 99.5 & 99.6 & 99.9 \\
            MixVPR~\cite{ali2023mixvpr} & 4,096 & 63.8 & 78.8 & 91.5 & 95.5 & 91.5 & 96.0 & 92.7 & 99.5 & 99.7 & \textbf{100.0} \\
            SALAD~\cite{izquierdo2024optimal} & 8,448 & 64.3 & 78.8 & 92.3 & 96.2 & 96.0 & 99.5 & 99.0 & \textbf{100.0} & \textbf{100.0} & \textbf{100.0} \\
            SelaVPR~\cite{lu2024towards} & 1,024 & 64.5 & 79.6 & 93.0 & 97.0 & 96.0 & \textbf{100.0} & 97.9 & \textbf{100.0} & 99.6 & 99.9 \\
            \textbf{DINOv2-RIA-FT (ours)} & 2,080 & \textbf{65.1} & \textbf{82.5} & \textbf{93.5} & \textbf{97.8} & \textbf{98.2} & \textbf{100.0} & \textbf{99.5} & \textbf{100.0} & \textbf{100.0} & \textbf{100.0} \\
            \midrule
            
            \multicolumn{12}{l}{\textit{Unsupervised Learning}} \\
            DINOv2-VLAD & 12,288 & 64.2 & 81.2 & 85.7 & 93.9 & 96.5 & \textbf{100.0} & 82.7 & 92.6 & 96.5 & 99.1 \\
            AnyLoc-VLAD-DINO~\cite{keetha2023anyloc} & 24,576 & 63.8 & 78.8 & 83.4 & 92.0 & 95.0 & 98.5 & 82.2 & 99.0 & 88.5 & 94.9 \\
            AnyLoc-VLAD-DINOv2~\cite{keetha2023anyloc} & 49,152 & 65.0 & 80.5 & 87.7 & 94.7 & 95.5 & 99.5 & \textbf{99.5} & \textbf{100.0} & 96.2 & 98.8 \\
            \midrule
            
            \multicolumn{12}{l}{\textit{Training-free}} \\
            DINOv2-GeM & 1,536 & 63.7 & 80.0 & 83.6 & 92.3 & 89.5 & 99.0 & 88.4 & 94.7 & 89.0 & 95.6 \\
 
            \textbf{DINOv2-RIA (ours)} & 2,080 & 64.7 & 81.2 & 86.7 & 93.8 & 97.5 & 99.5 & 98.4 & \textbf{100.0} & 97.2 & 98.7 \\
            \bottomrule
        \end{tabular}%
    }
\end{table*}

\begin{table*}[htbp]
    \centering
    \small
    \setlength{\tabcolsep}{10pt}
    \caption{\textbf{Comparison across unstructured VPR benchmarks.} }
    \label{tab:unstructured_results}
    
    \begin{tabular}{l c c c c c c c c c}
        \toprule
        \multirow{2}{*}{\textbf{Method}} & \multirow{2}{*}{\textbf{Dim.}} & \multicolumn{2}{c}{\textbf{Hawkins}} & \multicolumn{2}{c}{\textbf{Laurel}} & \multicolumn{2}{c}{\textbf{Nardo-Air}} & \multicolumn{2}{c}{\textbf{Mid-Atlantic}} \\
        \cmidrule(lr){3-4} \cmidrule(lr){5-6} \cmidrule(lr){7-8} \cmidrule(lr){9-10}
         & & R@1 & R@5 & R@1 & R@5 & R@1 & R@5 & R@1 & R@5 \\
        \midrule
        
        \multicolumn{10}{l}{\textit{Supervised Learning}} \\
        NetVLAD~\cite{arandjelovic2016netvlad} & 32,768 & 34.8 & 71.2 & 39.3 & 71.4 & 19.7 & 39.4 & 25.7 & 53.5 \\
        CosPlace~\cite{berton2022rethinking} & 512 & 31.4 & 59.3 & 24.1 & 47.3 & 0.0 & 1.4 & 40.6 & 20.8 \\
        MixVPR~\cite{ali2023mixvpr} & 4,096 & 27.1 & 58.5 & 31.3 & 64.3 & 33.8 & 43.7 & 25.7 & 59.4 \\
        SALAD~\cite{izquierdo2024optimal} & 8,448 & 38.1 & 77.1 & 52.7 & 81.3 & 38.0 & 76.1 & 14.9 & 46.5 \\
        SelaVPR~\cite{lu2024towards} & 1,024 & 27.1 & 72.9 & 58.0 & \textbf{90.2} & 46.5 & 85.9 & 24.8 & 57.4 \\
        \midrule
        
        \multicolumn{10}{l}{\textit{Unsupervised Learning}} \\
        AnyLoc-VLAD-DINO~\cite{keetha2023anyloc} & 24,576 & 48.3 & 84.8 & 57.1 & 79.5 & 43.7 & 54.9 & 41.6 & \textbf{66.3} \\
        AnyLoc-VLAD-DINOv2~\cite{keetha2023anyloc} & 49,152 & 65.2 & 94.1 & 61.6 & \textbf{90.2} & 76.1 & 94.4 & 34.6 & 61.4 \\
        \midrule
        
        \multicolumn{10}{l}{\textit{Training-free}} \\
        \textbf{DINOv2-RIA (ours)} & 2,080 & \textbf{66.7} & \textbf{95.2} & \textbf{65.2} & 85.7 & \textbf{84.5} & \textbf{100.0} & \textbf{41.8} & 58.4 \\
        \bottomrule
    \end{tabular}
\end{table*}

\begin{table}[htbp]
    \centering
    \small
    \caption{\textbf{Performance comparison of RIA head across various backbones in unstructured environments.} Each backbone is evaluated using the RIA aggregation head with a fixed descriptor dimension (Dim.) of 2080. We report Recall@1 / Recall@5.}
    \label{tab:backbone_ria_unstructured}
    \resizebox{\linewidth}{!}{%
        \begin{tabular}{l c cc cc cc}
            \toprule
            \multirow{2}{*}{\textbf{Method}} & \multirow{2}{*}{\textbf{Dim.}} & \multicolumn{2}{c}{\textbf{Hawkins}} & \multicolumn{2}{c}{\textbf{Laurel Caverns}} & \multicolumn{2}{c}{\textbf{Nardo-Air}} \\
            \cmidrule(lr){3-4} \cmidrule(lr){5-6} \cmidrule(lr){7-8}
            & & R@1 & R@5 & R@1 & R@5 & R@1 & R@5 \\
            \midrule
            MixVPR + RIA  & 2080 & 34.75 & 65.25 & 38.39 & 70.54 & 46.48 & 52.11 \\
            SALAD + RIA   & 2080 & \textbf{47.46} & 88.14 & 48.21 & 81.25 & \textbf{70.42} & 76.06 \\
            SelaVPR + RIA & 2080 & 36.44 & 77.12 & \textbf{66.96} & 91.07 & 28.17 & 92.96 \\
            \bottomrule
        \end{tabular}%
    }
\end{table}

\subsection{Compared Baselines}
To comprehensively evaluate the effectiveness of our proposed framework, we benchmark it against two distinct categories of state-of-the-art VPR methods: supervised approaches requiring extensive labeled data, and unsupervised methods leveraging pre-trained foundation models.
\par
\textbf{Supervised methods} learn representations from large-scale labeled datasets.
NetVLAD~\cite{arandjelovic2016netvlad} aggregates local descriptors via learnable cluster residuals.
CosPlace~\cite{berton2022rethinking} reformulates VPR as classification with GeM pooling.
MixVPR~\cite{ali2023mixvpr} employs MLP-based Feature-Mixers to capture holistic scene layout.
SALAD~\cite{izquierdo2024optimal} introduces optimal transport for soft feature assignment.
SelaVPR~\cite{lu2024towards} augments a frozen DINOv2 backbone with lightweight adapters trained on labeled data.
\par
\textbf{Unsupervised Learning Baselines.}
In the unsupervised regime, methods leverage frozen visual backbones without task-specific fine-tuning. We primarily benchmark against AnyLoc~\cite{keetha2023anyloc}, a robust general-purpose baseline that processes features from self-supervised Transformers (DINO and DINOv2) via PCA reduction and VLAD aggregation. We report results for both \textit{AnyLoc-VLAD-DINO} and \textit{AnyLoc-VLAD-DINOv2}. Crucially, although VLAD achieves strong unsupervised performance through the use of domain-specific vocabularies, it still incurs a high clustering cost and therefore cannot be considered truly plug-and-play. In contrast, our method does not rely on any form of clustering.

\subsection{Implementation Details}

\label{sec:implementation}
We utilize the pre-trained DINOv2-Giant~\cite{oquab2023dinov2} (frozen) as the backbone. Following AnyLoc~\cite{keetha2023anyloc}, dense local descriptors are extracted from the 31st Transformer layer. Input images are resized to multiples of the $14 \times 14$ patch size. Feature aggregation employs a random projection to a lower-dimensional subspace to ensure full-rank covariance estimation. Structural rectification utilizes $\tau=10^{-5}$ and $\epsilon=10^{-4}$. Riemannian linearization is performed via $K=3$ Newton-Schulz iterations. All experiments are conducted on an NVIDIA A100 GPU.
\par
To test the supervised performance of our method, we convert the original training-free RIA head into a trainable aggregation module by replacing the fixed random projection with learnable projections, extending the single-branch covariance descriptor to a 4 head design, and adding lightweight per-head MLP compression. The resulting head retains the original RIA geometry pipeline while being trained jointly with the last four unfrozen blocks of DINOv2-Large on GSV-Cities.

\paragraph{Datasets.}
We evaluate on nine benchmarks spanning both structured and unstructured environments.
\textbf{Structured benchmarks} cover standard urban and indoor scenarios:
17places~\cite{sahdev2016indoor} (indoor),
Pitts30k~\cite{torii2013visual} (urban street-view),
Gardens~\cite{glover2014gardens} (campus with seasonal change),
Oxford RobotCar~\cite{maddern20171} (long-term urban with day-night and seasonal shifts), and
St.~Lucia~\cite{warren2010unaided} (suburban driving).
\textbf{Unstructured benchmarks}, drawn from the AnyLoc suite~\cite{keetha2023anyloc}, test robustness in visually degraded and out-of-distribution settings:
Hawkins (featureless subterranean corridors),
Laurel Caverns (low-illumination cave environments),
Nardo-Air (aerial-to-satellite with extreme viewpoint shift), and
Mid-Atlantic Ridge (underwater terrain with limited texture).
All results are reported using Recall@1 and Recall@5.

\subsection{Main Results}
\begin{table*}[ht]
    \centering
    \caption{\textbf{Analysis of Statistical Order and Manifold Geometry.} We compare representations with different statistical orders (1st vs. 2nd) and geometric mappings (Euclidean vs. Riemannian).}
    \label{tab:ablation_geometry}
    \resizebox{0.9\textwidth}{!}{
    \begin{tabular}{lccccc}
        \toprule
        \textbf{Representation Method} & \textbf{Statistical Order} & \textbf{Geometric Mapping} & \textbf{Operator} & \textbf{Pitts30k R@1} & \textbf{Nardo-Air R@1} \\
        \midrule
        Magnitude Pooling & 1st Order & Euclidean & Mean (GeM) & 83.6 & 71.8 \\
        Euclidean Covariance & 2nd Order & Euclidean & Identity & 82.2 & 70.4 \\
        Log-Euclidean Covariance & 2nd Order & Riemannian & Matrix Log & 86.1 & 63.4 \\
        \midrule
        \textbf{Power-Euclidean Covariance} & \textbf{2nd Order} & \textbf{Riemannian} & \textbf{Matrix Sqrt} & \textbf{86.7} & \textbf{76.1} \\
        \bottomrule
    \end{tabular}
    }
\end{table*}

\paragraph{Structured environments (Tab.~\ref{tab:main_results}).}
On standard urban and indoor benchmarks, our supervised variant DINOv2-RIA-FT achieves the highest R@1 on all five datasets, outperforming the previous best supervised methods SALAD and SelaVPR. This validates that the RIA geometry pipeline remains effective when extended with learnable components.
 
More notably, the \emph{training-free} DINOv2-RIA already matches or exceeds most supervised baselines on cross-domain benchmarks. On Oxford and Gardens---where illumination and seasonal shifts dominate---RIA reaches 98.4\% and 97.5\% R@1, surpassing MixVPR by +5.7\% and +6.0\% respectively, despite using no training data. On Pitts30k, where supervised methods benefit from strong domain alignment with their urban training set (GSV-Cities), a gap persists (86.7\% vs.\ 93.0\% for SelaVPR). This confirms that geometric invariance is most advantageous under distribution shift, while in-domain performance still favors data-driven fitting.
 
Compared to first-order baselines sharing the same backbone, RIA consistently dominates: +10.0\% over DINOv2-GeM on Oxford, +8.0\% on Gardens, and +8.2\% on St.~Lucia. This demonstrates that covariance modeling captures discriminative structural correlations discarded by mean pooling.
 
\paragraph{Unstructured environments (Tab.~\ref{tab:unstructured_results}).}
The advantage of geometric aggregation becomes pronounced in out-of-distribution scenarios. Supervised methods suffer severe degradation---CosPlace scores 0.0\% on Nardo-Air, and MixVPR drops to 27.1\% on Hawkins---since their learned urban priors do not transfer to subterranean, aerial, or underwater domains.

DINOv2-RIA achieves the best R@1 on Hawkins (66.7\%), Laurel Caverns (65.2\%), Nardo-Air (84.5\%), and matches the top result on Mid-Atlantic Ridge (41.8\%), all without any data-dependent preprocessing. Notably, AnyLoc-DINOv2 attains comparable Hawkins performance (65.2\%) only by pre-computing a domain-specific VLAD vocabulary via database clustering, coupling its deployment to the availability of reference data. In contrast, RIA operates as a truly plug-and-play module---no vocabulary construction, no clustering, and no access to the database distribution is needed. Despite this strictly zero-resource setting, RIA delivers competitive or superior performance across all unstructured benchmarks, confirming that intrinsic geometric structure provides a more general-purpose matching signal than vocabulary-engineered first-order aggregation.

 \paragraph{Backbone generality (Tab.~\ref{tab:backbone_ria_unstructured}).}
To verify that RIA's benefit is not tied to a specific backbone, we attach the same RIA head to features extracted by MixVPR, SALAD, and SelaVPR. As shown in Tab.~\ref{tab:backbone_ria_unstructured}, RIA consistently improves all three backbones in unstructured environments, with SALAD+RIA achieving the strongest overall results. This confirms that the geometric aggregation provided by RIA is complementary to diverse feature extractors and generalizes beyond the DINOv2 backbone.

\par

\subsection{Impact of Statistical Order and Manifold Geometry}
Tab.~\ref{tab:ablation_geometry} isolates the contributions of statistical order and geometric mapping on Pitts30k (structured) and Nardo-Air (unstructured).

Naively lifting features to second-order statistics without geometric correction (Euclidean Covariance) \emph{hurts} performance relative to first-order pooling on both benchmarks (82.2\% vs.\ 83.6\% on Pitts30k). This confirms the well-known swelling effect: treating the SPD cone as flat Euclidean space distorts intrinsic distances between covariance descriptors.

Introducing Riemannian geometry recovers the benefit of second-order modeling, but the choice of mapping is critical. Log-Euclidean Covariance achieves the highest Pitts30k R@1 among non-PEM variants (86.1\%), yet collapses on Nardo-Air to 63.4\%---well below even first-order GeM (71.8\%). The matrix logarithm amplifies small eigenvalue differences, making it sensitive to the ill-conditioned covariance matrices common in visually degraded environments.

 Power-Euclidean Covariance resolves this trade-off. The matrix square root ($\alpha{=}0.5$) applies a milder spectral nonlinearity that flattens manifold curvature without over-amplifying noise. It achieves the best R@1 on both Pitts30k (86.7\%) and Nardo-Air (76.1\%), demonstrating consistent robustness across structured and unstructured domains.

\subsection{Effect of Metric Deformation via the Power Parameter}
\label{app:Power Parameter}
To elucidate the geometric mechanism behind our representation, we investigate the sensitivity of the framework to the power parameter $\alpha \in (0, 1]$. This parameter controls the degree of metric deformation applied to the Symmetric Positive Definite (SPD) manifold, effectively interpolating between the Euclidean geometry ($\alpha=1$) and the Riemannian Log-Euclidean geometry ($\alpha \to 0$).

\begin{table}[ht]
    \centering
    \caption{\textbf{Impact of the Power Parameter $\alpha$ on Pitts30k and Nardo-Air.} The parameter $\alpha$ controls the degree of metric deformation, interpolating between the Euclidean geometry ($\alpha=1.0$) and the Log-Euclidean geometry ($\alpha \to 0$).}
    \label{tab:power_ablation}

    \resizebox{0.4\textwidth}{!}{
        \begin{tabular}{l cc cc}
            \toprule
            \multirow{2}{*}{\textbf{Power Parameter}} & \multicolumn{2}{c}{\textbf{Pitts30k}} & \multicolumn{2}{c}{\textbf{Nardo-Air}} \\
            \cmidrule(lr){2-3} \cmidrule(lr){4-5}
            & \textbf{R@1} & \textbf{R@5} & \textbf{R@1} & \textbf{R@5} \\
            \midrule
            $\alpha = 1.0$ (Euclidean) & 82.2 & 92.6 & 70.4 & \textbf{100.0} \\
            $\alpha = 0.75$ & 83.5 & 93.2 & 75.5 & \textbf{100.0} \\
            \textbf{$\alpha = 0.5$ (Ours)} & \textbf{86.7} & 93.8 & \textbf{76.1} & \textbf{100.0} \\
            $\alpha = 0.25$ & 85.6 & 93.7 & 64.8 & 98.6 \\
            $\alpha = 0.1$ & 86.0 & 93.8 & 54.9 & 97.2 \\
            $\alpha \to 0$ (Log-Euclidean) & 86.1 & \textbf{94.0} & 63.4 & 98.6 \\
            \bottomrule
        \end{tabular}
    }
\end{table}

As shown in Tab.~\ref{tab:power_ablation}, treating SPD matrices as points in a flat Euclidean space ($\alpha=1.0$) yields the lowest performance (82.2\% R@1), highlighting the necessity of geometric rectification. Decreasing $\alpha$ progressively flattens the manifold curvature, leading to consistent performance gains. Notably, our choice of $\alpha=0.5$ (matrix square root) achieves the optimal R@1 of 86.7\%, slightly outperforming the Log-Euclidean limit ($\alpha \to 0$, 86.1\%). This suggests that the matrix square root provides a superior trade-off, offering sufficient geometric linearization to approximate geodesic distances while maintaining better numerical stability and noise robustness than the matrix logarithm for high-dimensional descriptors.

\begin{figure}[htbp]
  \centering
  \includegraphics[width=0.45\textwidth]{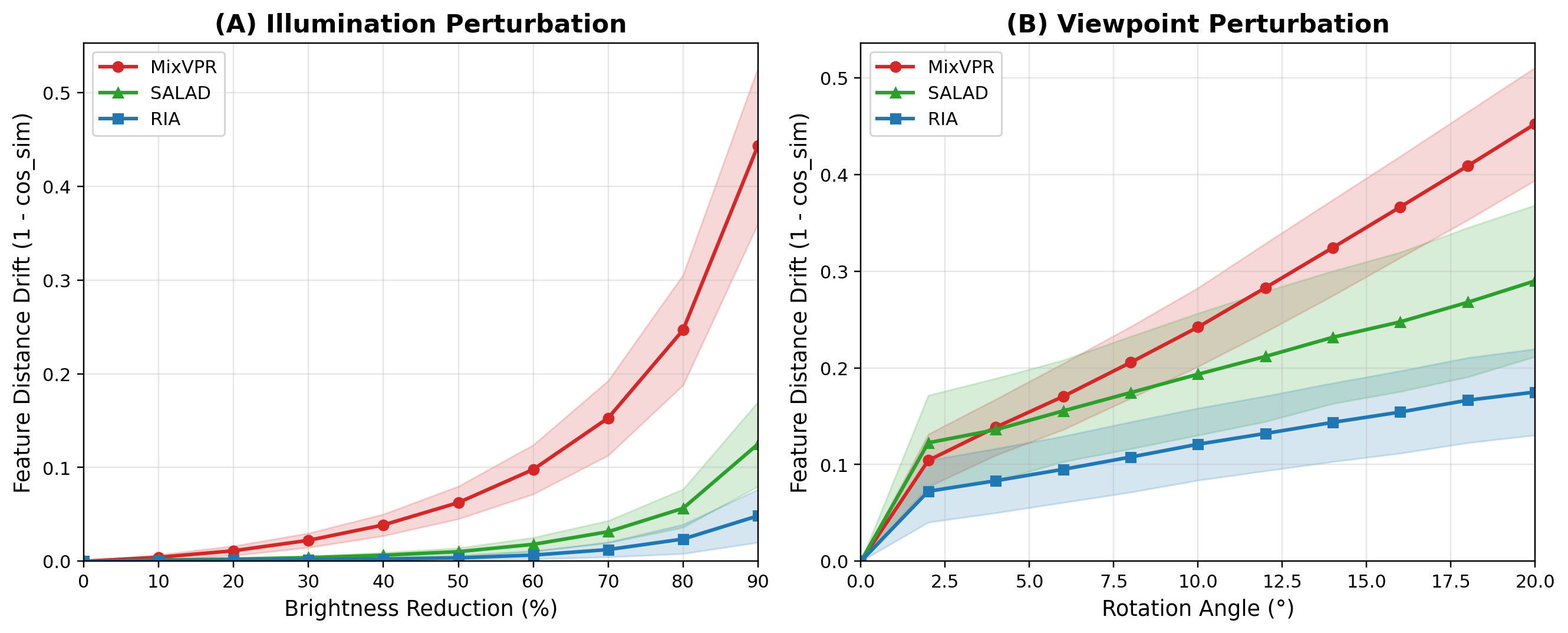}

  \includegraphics[width=0.45\textwidth]{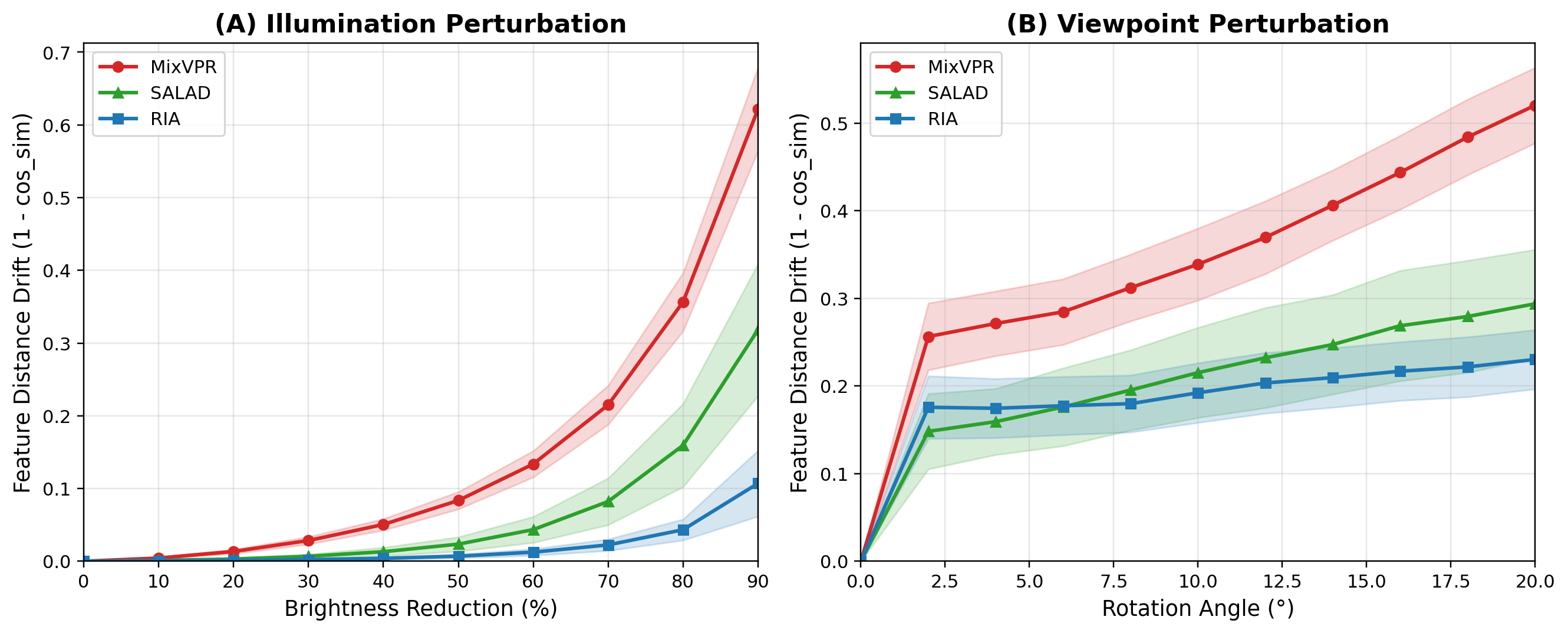}
  
\caption{\textbf{Feature distance drift under synthetic perturbations.} We measure $1 - \text{cosine\_similarity}$ between original and perturbed descriptors on Pitts30k-test (top) and Nardo-Air (bottom) under brightness reduction (left) and planar rotation (right).}
  \label{fig:invariance}
\end{figure}

\subsection{Ablation on Manifold Stabilization and Computational Efficiency}

We ablate the stabilization components of our pipeline on Pitts30k and Nardo-Air, fixing the backbone, $d{=}64$, and $\alpha{=}0.5$ throughout.

As shown in Tab.~\ref{tab:stability_ablation}, sparse rectification with $\tau{=}10^{-5}$ provides consistent minor gains by suppressing spurious off-diagonal correlations. Light SPD regularization ($\epsilon{=}10^{-6}$) further stabilizes the spectrum, but raising $\epsilon$ to $10^{-4}$ with the exact eigen-solver degrades performance, as the aggressive jitter over-smooths discriminative spectral structure. Replacing the exact solver with $K{=}3$ Newton--Schulz iterations under the same $\epsilon{=}10^{-4}$ recovers and surpasses all previous variants. We attribute this to the implicit spectral smoothing of the finite-step approximation, which complements rather than amplifies the explicit regularization.
\begin{table}[ht]
    \centering
    \caption{\textbf{Ablation of stabilization and matrix square-root implementation.}
    All variants use the same backbone features, projection dimension $d=64$, and power parameter $\alpha=0.5$.}
    \label{tab:stability_ablation}
    \resizebox{\linewidth}{!}{
    \begin{tabular}{l c c c c c c}
        \toprule
        \textbf{Variant} & $\tau$ (ReCov) & $\epsilon$ (ReEig) & \textbf{Solver} & \textbf{NS Iter} & \textbf{Pitts30k R@1} & \textbf{Nardo-Air R@1} \\
        \midrule
        Base PE-Cov                   & --         & --         & Exact & -- & 85.8 & 82.4 \\
        + Rectification (mild)        & $10^{-6}$  & --         & Exact & -- & 85.9 & 82.4 \\
        + Rectification (default)     & $10^{-5}$  & --         & Exact & -- & 86.1 & 82.9 \\
        + SPD Regularization          & $10^{-5}$  & $10^{-6}$  & Exact & -- & 86.2 & 83.1 \\
        + Stronger SPD Regularization & $10^{-5}$  & $10^{-4}$  & Exact & -- & 85.3 & 80.8 \\
        \textbf{Ours}                 & $\mathbf{10^{-5}}$ & $\mathbf{10^{-4}}$ & \textbf{NS} & \textbf{3} & \textbf{86.7} & \textbf{84.5} \\
        \bottomrule
    \end{tabular}
    }
\end{table}

\begin{table}[ht]
    \centering
    \caption{\textbf{Projection dimension and Newton--Schulz trade-off.}
    We fix covariance rectification and SPD regularization to the best settings from Tab.~\ref{tab:stability_ablation}, and vary the projection dimension and matrix square-root approximation.
    Speedup is relative to the \textbf{Exact, $d=64$} setting on \textbf{Pitts30k} (higher is faster).}
    \label{tab:proj_ns_tradeoff}
    \resizebox{\linewidth}{!}{
    \begin{tabular}{ccccccc}
        \toprule
        \textbf{Proj. Dim $d$} & \textbf{Output Dim} & \textbf{Solver} & \textbf{NS Iter} & \textbf{Pitts30k R@1} & \textbf{Nardo-Air R@1} & \textbf{Speedup} \\
        \midrule
        32  & 528  & Exact & -- & 82.5 & 70.5 & 1.43$\times$ \\
        64  & 2080 & Exact & -- & 85.3 & 80.8 & 1.00$\times$ \\
        128 & 8256 & Exact & -- & 86.0 & 83.5 & 0.66$\times$ \\
        \midrule
        64  & 2080 & NS & 1 & 84.2 & 78.6 & 1.54$\times$ \\
        \textbf{64} & \textbf{2080} & \textbf{NS} & \textbf{3} & \textbf{86.7} & \textbf{84.5} & \textbf{1.39$\times$} \\
        64  & 2080 & NS & 5 & 86.1 & 83.0 & 1.25$\times$ \\
        \bottomrule
    \end{tabular}
    }
\end{table}

\begin{figure*}[ht]
  \centering
  \includegraphics[width=0.8\textwidth]{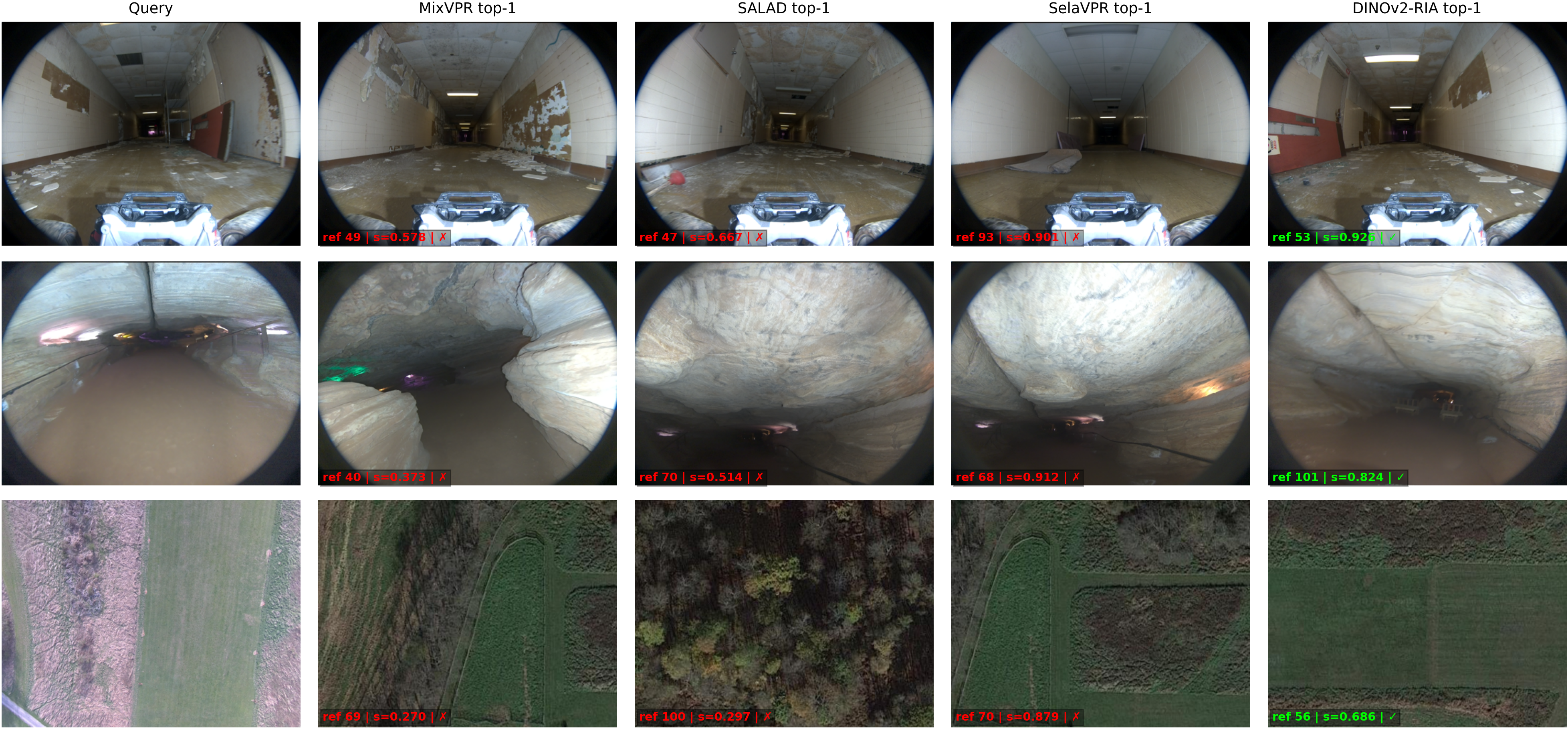}
\caption{\textbf{Qualitative top-1 retrieval on unstructured benchmarks.} From top to bottom: Hawkins, Laurel Caverns, and Mid-Atlantic Ridge. Each row shows the query (left) and the top-1 result from MixVPR, SALAD, SelaVPR, and DINOv2-RIA (ours).}
\label{visual}
\end{figure*}
 
Tab.~\ref{tab:proj_ns_tradeoff} further explores the projection dimension and iteration count. Among exact-solver variants, larger $d$ improves accuracy at the cost of cubic runtime growth. With NS at $d{=}64$, three iterations achieve the best recall on both benchmarks while running 1.39$\times$ faster than the exact baseline---outperforming even the higher-capacity $d{=}128$ exact variant. Increasing to five iterations diminishes the implicit regularization benefit as the approximation converges toward the exact solution. These results confirm that NS offers a favorable accuracy--efficiency trade-off, where its finite-step nature acts as a beneficial regularizer.

\subsection{Empirical Verification of Structural Invariance}
To validate geometric stability, we compare feature distance drift ($1 - \text{cosine\_similarity}$) of RIA against MixVPR~\cite{ali2023mixvpr} and SALAD~\cite{izquierdo2024optimal} under progressive brightness reduction and planar rotation on Pitts30k-test and Nardo-Air (Fig.~\ref{fig:invariance}).

Under illumination perturbation, MixVPR exhibits the largest drift on both datasets, while SALAD and RIA remain comparatively stable. RIA achieves the lowest drift overall, consistent with the homogeneity of the Power-Euclidean Metric which naturally neutralizes global intensity scalings ($s^2 \boldsymbol{C}$).

Under viewpoint perturbation, the advantage of geometric modeling becomes more pronounced. MixVPR and SALAD both degrade substantially as rotation increases, with the gap widening on the more challenging Nardo-Air dataset. RIA consistently maintains the lowest drift, as our metric treats rotation as a congruence transformation ($\boldsymbol{Q}\boldsymbol{C}\boldsymbol{Q}^\top$), decoupling intrinsic scene geometry from extrinsic camera pose---a structural guarantee that learned first-order representations cannot provide.

\subsection{Qualitative Retrieval Examples}
Fig.~\ref{visual} shows top-1 retrieval results of MixVPR, SALAD, SelaVPR, and our DINOv2-RIA on three representative unstructured benchmarks: Hawkins (subterranean corridor), Laurel Caverns (low illumination cave), and Mid-Atlantic Ridge (aerial terrain). These environments pose extreme challenges for VPR: repetitive textures in featureless corridors, severe illumination degradation underground, and drastic viewpoint shifts between aerial and satellite imagery---all far removed from the urban street-view distribution on which supervised methods are trained.

As illustrated, supervised baselines frequently retrieve incorrect references, as their learned representations struggle to generalize across such large domain gaps. In contrast, RIA produces more reliable matches, suggesting that modeling second-order geometric structure provides a stronger cross-domain matching signal than data-driven first-order aggregation in these out-of-distribution scenarios.

\par
Crucially, RIA provides larger relative gains for weaker backbones. For instance, on the Baidu dataset, RIA improves the legacy ResNet50 performance by \textbf{+9.1\%} over VLAD (58.1\% $\to$ 67.2\%), whereas the gain on the stronger DINOv2 is more moderate (+3.8\%). This indicates that when local features lack semantic robustness, RIA's explicit geometric modeling effectively compensates for backbone deficiencies, serving as a powerful structural prior that significantly boosts less discriminative models.

\section{Conclusion}
We presented Riemannian Invariant Aggregation (RIA), a unified geometric framework for VPR that models scenes as covariance descriptors on the SPD manifold. By exploiting the congruence invariance of second-order statistics and Riemannian-aware tangent-space projection, RIA achieves zero-shot performance competitive with supervised methods under large domain shifts, and establishes state-of-the-art accuracy with simple fine-tuning. Experiments across structured and unstructured benchmarks confirm that geometric modeling of intrinsic scene structure offers a powerful, general-purpose alternative to data-driven aggregation.
\bibliographystyle{ACM-Reference-Format}
\bibliography{arxiv}

\clearpage
\appendix
\section*{Supplementary Material}
\addcontentsline{toc}{section}{Supplementary Material}
\section*{Overview}
This supplementary material provides additional theoretical analysis, extended ablation studies, implementation details, and qualitative results that complement the main paper. The content is organized as follows:
\begin{itemize}
    \item \textbf{Section 1}: Notations and abbreviations.
    \item \textbf{Section 2}: Theoretical analysis of invariance properties, isometric vectorization, and Newton--Schulz convergence.
    \item \textbf{Section 3}: Extended ablation studies on random projection, hyperparameter sensitivity, covariance rectification, and fine-tuning components.
    \item \textbf{Section 4}: Implementation and reproducibility details for the fine-tuned model and backbone-swap evaluation setup.
    \item \textbf{Section 5}: Additional qualitative analysis including extended perturbation studies and failure cases.
\end{itemize}

\section{Notations and Abbreviations}
\label{app:notations}

\begin{table}[ht]
    \centering
    \caption{Summary of Mathematical Notations.}
    \label{tab:notations}
    \begin{tabular}{cl}
        \toprule
        \textbf{Notation} & \textbf{Explanation} \\
        \midrule
        \multicolumn{2}{l}{\textit{Spaces and Sets}} \\
        $\mathcal{S}_{++}^d$ & The Riemannian manifold of $d \times d$ SPD matrices \\
        $\text{Sym}(d)$ & The vector space of $d \times d$ symmetric matrices \\
        $O(d)$ & The Orthogonal Group of $d \times d$ matrices ($\boldsymbol{Q}^\top\boldsymbol{Q} = \boldsymbol{I}$) \\
        \midrule
        \multicolumn{2}{l}{\textit{Scalars and Hyperparameters}} \\
        $D_{\text{in}}$ & Native feature dimension of the backbone  \\
        $N$ & Number of local patches extracted from an image \\
        $d$ & Dimension of the projected geometric subspace \\
        $D$ & Dimension of the final global descriptor, $D = d(d+1)/2$ \\
        $\tau$ & Structural saliency threshold for ReCov \\
        $\epsilon$ & Regularization constant for SPD regularization \\
        $K$ & Number of Newton--Schulz iterations \\
        $\alpha$ & Power parameter for PEM \\
        $s$ & Scalar factor representing global illumination intensity \\
        \midrule
        \multicolumn{2}{l}{\textit{Vectors and Matrices}} \\
        $\boldsymbol{I}_d$ & The $d \times d$ identity matrix \\
        $\boldsymbol{x}_i$ & The $i$-th local feature vector \\
        $\bar{\boldsymbol{x}}$ & The mean vector of local features \\
        $\boldsymbol{X}_{\text{raw}}$ & The raw feature matrix from backbone, $\in \mathbb{R}^{N \times D_{\text{in}}}$ \\
        $\boldsymbol{P}$ & The random orthogonal projection matrix, $\in \mathbb{R}^{D_{\text{in}} \times d}$ \\
        $\boldsymbol{C}_{\text{raw}}$ & The sample covariance matrix \\
        $\boldsymbol{C}_{\text{rec}}$ & The rectified covariance matrix after ReCov \\
        $\boldsymbol{C}$ & The regularized SPD descriptor, $\boldsymbol{C} \in \mathcal{S}_{++}^d$ \\
        $\boldsymbol{A}_0$ & The pre-normalized matrix for NS initialization \\
        $\boldsymbol{Y}_k, \boldsymbol{Z}_k$ & NS iterates approximating $\boldsymbol{A}_0^{1/2}$ and $\boldsymbol{A}_0^{-1/2}$ \\
        $\boldsymbol{M}$ & The mapped descriptor on the linearized tangent space \\
        $\boldsymbol{v}$ & The final isometric vectorized global descriptor \\
        $\boldsymbol{Q}$ & Orthogonal matrix representing viewpoint transformation \\
        \midrule
        \multicolumn{2}{l}{\textit{Operators and Functions}} \\
        $\|\cdot\|_F$ & The Frobenius norm of a matrix \\
        $\|\cdot\|_2$ & The Euclidean ($L_2$) norm of a vector \\
        $\text{tr}(\cdot)$ & The trace operator \\
        $\langle \cdot, \cdot \rangle_F$ & The Frobenius inner product, $\text{tr}(\boldsymbol{A}^\top \boldsymbol{B})$ \\
        $\text{vec}(\cdot)$ & The isometric vectorization operator (with $\sqrt{2}$ weighting) \\
        $\Phi(\cdot)$ & The complete RIA descriptor pipeline \\
        \bottomrule
    \end{tabular}
\end{table}

\begin{table}[ht]
    \centering
    \caption{Summary of Abbreviations.}
    \label{tab:abbreviations}
    \begin{tabular}{cl}
        \toprule
        \textbf{Abbreviation} & \textbf{Explanation} \\
        \midrule
        VPR & Visual Place Recognition \\
        SPD & Symmetric Positive Definite \\
        RIA & Riemannian Invariant Aggregation \\
        ReCov & Rectified Covariance \\
        PEM & Power Euclidean Metric \\
        LEM & Log-Euclidean Metric \\
        NS & Newton--Schulz iteration \\
        EIG & Eigenvalue Decomposition \\
        GeM & Generalized Mean pooling \\
        VLAD & Vector of Locally Aggregated Descriptors \\
        \bottomrule
    \end{tabular}
\end{table}

\section{Theoretical Analysis}
\label{app:theory}

\subsection{Formal Invariance Properties}
\label{app:invariance_formal}

We formalize the invariance properties of the RIA descriptor under explicit assumptions about how environmental perturbations affect the feature space. Let $\boldsymbol{X} = [\boldsymbol{x}_1, \dots, \boldsymbol{x}_N]^\top \in \mathbb{R}^{N \times d}$ denote the projected patch features, and $\boldsymbol{C} = \frac{1}{N-1}\sum_{i}(\boldsymbol{x}_i - \bar{\boldsymbol{x}})(\boldsymbol{x}_i - \bar{\boldsymbol{x}})^\top \in \mathcal{S}_{++}^d$ denote the sample covariance.

\subsubsection{Illumination Invariance}

\begin{assumption}[Illumination Model]
\label{ass:illumination}
A global illumination change acts as a scalar multiplication on the feature vectors: $\tilde{\boldsymbol{x}}_i = s \boldsymbol{x}_i$ for some $s > 0$. This induces a quadratic scaling of the covariance matrix: $\tilde{\boldsymbol{C}} = s^2 \boldsymbol{C}$.
\end{assumption}

\begin{proposition}[Illumination Invariance]
\label{prop:illumination}
Let $\Phi(\boldsymbol{C})$ denote the complete RIA descriptor pipeline (PEM mapping with power $\alpha$ $\to$ isometric vectorization $\to$ $L_2$ normalization). Under Assumption~\ref{ass:illumination}, the final descriptor is exactly invariant to global intensity scaling:
\begin{equation}
    \Phi(s^2 \boldsymbol{C}) = \Phi(\boldsymbol{C}), \quad \forall\, s > 0.
\end{equation}
\end{proposition}
\begin{proof}
We trace the transformation of $\tilde{\boldsymbol{C}} = s^2 \boldsymbol{C}$ through each pipeline stage:
\begin{enumerate}
    \item \textbf{PEM Mapping.} The matrix power function is positively homogeneous: $\tilde{\boldsymbol{C}}^{\alpha} = (s^2 \boldsymbol{C})^{\alpha} = s^{2\alpha} \boldsymbol{C}^{\alpha}$.
    \item \textbf{Isometric Vectorization.} The vectorization operator is linear, so $\tilde{\boldsymbol{v}} = \operatorname{vec}(s^{2\alpha} \boldsymbol{C}^{\alpha}) = s^{2\alpha} \operatorname{vec}(\boldsymbol{C}^{\alpha}) = s^{2\alpha} \boldsymbol{v}$.
    \item \textbf{$L_2$ Normalization.} The positive scalar $s^{2\alpha}$ cancels:
    \begin{equation}
        \Phi(\tilde{\boldsymbol{C}}) = \frac{\tilde{\boldsymbol{v}}}{\|\tilde{\boldsymbol{v}}\|_2} = \frac{s^{2\alpha} \boldsymbol{v}}{s^{2\alpha} \|\boldsymbol{v}\|_2} = \frac{\boldsymbol{v}}{\|\boldsymbol{v}\|_2} = \Phi(\boldsymbol{C}).
    \end{equation}
\end{enumerate}
\end{proof}

\begin{remark}[Advantage over Log-Euclidean Metric]
\label{rem:pem_vs_lem}
This invariance is a practical advantage of PEM over the Log-Euclidean Metric (LEM). Under the same scaling, LEM produces $\log(s^2 \boldsymbol{C}) = \log(\boldsymbol{C}) + 2\ln(s)\,\boldsymbol{I}_d$. The additive term $2\ln(s)\,\boldsymbol{I}_d$ shifts all diagonal entries uniformly and \emph{cannot} be removed by $L_2$ normalization alone, making LEM inherently sensitive to illumination intensity unless explicitly centered. PEM handles this naturally via homogeneity of the matrix power.
\end{remark}

\subsubsection{Viewpoint Invariance}

\begin{assumption}[Viewpoint Model]
\label{ass:viewpoint}
A change in viewpoint is modeled as an orthogonal transformation in the projected feature space: $\tilde{\boldsymbol{x}}_i = \boldsymbol{Q}\boldsymbol{x}_i$, where $\boldsymbol{Q} \in O(d)$ satisfies $\boldsymbol{Q}^\top\boldsymbol{Q} = \boldsymbol{I}_d$. This induces a congruence transformation on the covariance matrix: $\tilde{\boldsymbol{C}} = \boldsymbol{Q}\boldsymbol{C}\boldsymbol{Q}^\top$.
\end{assumption}

\begin{proposition}[Viewpoint Distance Invariance]
\label{prop:viewpoint}
For any $\boldsymbol{C}_1, \boldsymbol{C}_2 \in \mathcal{S}_{++}^d$, any orthogonal $\boldsymbol{Q} \in O(d)$, and any power $\alpha \in (0, 1]$, the PEM distance is invariant under congruence:
\begin{equation}
    d_{\mathrm{PEM}}(\boldsymbol{Q}\boldsymbol{C}_1\boldsymbol{Q}^\top,\; \boldsymbol{Q}\boldsymbol{C}_2\boldsymbol{Q}^\top) = d_{\mathrm{PEM}}(\boldsymbol{C}_1, \boldsymbol{C}_2).
\end{equation}
\end{proposition}
\begin{proof}
The matrix power function is equivariant under orthogonal conjugation:
\begin{equation}
    (\boldsymbol{Q}\boldsymbol{C}\boldsymbol{Q}^\top)^{\alpha} = \boldsymbol{Q}\,\boldsymbol{C}^{\alpha}\,\boldsymbol{Q}^\top.
\end{equation}
Substituting into the PEM distance:
\begin{align}
    d_{\mathrm{PEM}}^2(\tilde{\boldsymbol{C}}_1, \tilde{\boldsymbol{C}}_2)
    &= \frac{1}{\alpha^2}\| \boldsymbol{Q}\boldsymbol{C}_1^{\alpha}\boldsymbol{Q}^\top - \boldsymbol{Q}\boldsymbol{C}_2^{\alpha}\boldsymbol{Q}^\top \|_F^2 \\
    &= \frac{1}{\alpha^2}\| \boldsymbol{Q}(\boldsymbol{C}_1^{\alpha} - \boldsymbol{C}_2^{\alpha})\boldsymbol{Q}^\top \|_F^2.
\end{align}
By the unitary invariance of the Frobenius norm ($\|\boldsymbol{U}\boldsymbol{A}\boldsymbol{V}\|_F = \|\boldsymbol{A}\|_F$ for orthogonal $\boldsymbol{U}, \boldsymbol{V}$):
\begin{equation}
    \| \boldsymbol{Q}(\boldsymbol{C}_1^{\alpha} - \boldsymbol{C}_2^{\alpha})\boldsymbol{Q}^\top \|_F^2 = \| \boldsymbol{C}_1^{\alpha} - \boldsymbol{C}_2^{\alpha} \|_F^2 = \alpha^2\, d_{\mathrm{PEM}}^2(\boldsymbol{C}_1, \boldsymbol{C}_2).
\end{equation}
\end{proof}

\begin{remark}
This result holds for \emph{all} $\alpha \in (0,1]$, not just $\alpha = 0.5$. The PEM mapping introduces a spectral nonlinearity ($\boldsymbol{C}^{\alpha}$) that flattens the manifold curvature without breaking rotational invariance, ensuring that the relative ranking among database descriptors is preserved under viewpoint shifts even when retrieval is performed in Euclidean space.
\end{remark}

\subsubsection{Second-Order Discriminability}

\begin{proposition}[Discriminability beyond First-Order Statistics]
\label{prop:discriminability}
There exist distinct scene feature distributions whose first-order means are identical ($\boldsymbol{\mu}_A = \boldsymbol{\mu}_B$) yet whose PEM distance on the SPD manifold is strictly positive.
\end{proposition}
\begin{proof}
Consider two scenes with centered features ($\boldsymbol{\mu}_A = \boldsymbol{\mu}_B = \mathbf{0}$) but distinct principal directions (e.g., vertical vs.\ horizontal texture). Let their covariance matrices be:
\begin{equation}
    \boldsymbol{C}_A = \text{diag}(\sigma_h, \sigma_l), \quad \boldsymbol{C}_B = \text{diag}(\sigma_l, \sigma_h), \quad \sigma_h \neq \sigma_l > 0.
\end{equation}
The Euclidean distance between their means is $\|\boldsymbol{\mu}_A - \boldsymbol{\mu}_B\|_2 = 0$, so first-order representations are indistinguishable. However, the PEM distance is:
\begin{align}
    d_{\mathrm{PEM}}(\boldsymbol{C}_A, \boldsymbol{C}_B) &= 2\| \boldsymbol{C}_A^{1/2} - \boldsymbol{C}_B^{1/2} \|_F = 2\sqrt{2}\,|\sqrt{\sigma_h} - \sqrt{\sigma_l}| > 0.
\end{align}
The second-order representation on the SPD manifold successfully discriminates the two scenes.
\end{proof}

\begin{remark}
This constructive example illustrates the perceptual aliasing problem: scenes with similar average appearance but different structural layouts (e.g., a corridor with horizontal beams vs.\ vertical pipes) share the same mean feature but have distinct covariance signatures. Second-order modeling on the SPD manifold captures the ``shape'' of feature distributions, resolving ambiguities that first-order methods cannot.
\end{remark}

\subsubsection{Approximation Regime and Pipeline Effects}

\begin{remark}[Applicability to Deep Features]
\label{rem:approximation}
Assumptions~\ref{ass:illumination} and~\ref{ass:viewpoint} are idealized models. For deep ViT features, environmental perturbations do not induce exact scalar/orthogonal transformations in the feature space. However, modern self-supervised features (e.g., DINOv2) are trained with augmentation-based objectives that encourage approximate equivariance to geometric and photometric transformations. Consequently, the invariance properties hold as \emph{first-order approximations} in the high-dimensional feature space. The empirical validation in Fig.~3 of the main paper and the extended perturbation study in Appendix~\ref{app:homography} demonstrate that these approximations are tight enough to yield practically meaningful robustness gains.
\end{remark}

\begin{remark}[Effect of ReCov and $\epsilon$-Regularization on Invariance]
\label{rem:pipeline_effects}
The idealized proofs above treat the raw sample covariance $\boldsymbol{C}$. In our actual pipeline, two additional operations precede the PEM mapping:
\begin{itemize}
    \item \textbf{ReCov} ($\mathcal{R}_\tau$): Hard thresholding of off-diagonal entries. Under a pure congruence $\boldsymbol{C} \mapsto \boldsymbol{Q}\boldsymbol{C}\boldsymbol{Q}^\top$, the thresholding is not strictly equivariant since it operates element-wise. However, when $\tau$ is small (our default $\tau = 10^{-5}$), only near-zero spurious correlations are removed, and the dominant structural entries that determine retrieval ranking are preserved.
    \item \textbf{$\epsilon$-Regularization}: Adding $\epsilon \boldsymbol{I}_d$ breaks exact scale invariance because $s^2 \boldsymbol{C} + \epsilon \boldsymbol{I}_d \neq s^2 (\boldsymbol{C} + \epsilon \boldsymbol{I}_d)$. With our small default $\epsilon = 10^{-4}$ relative to typical eigenvalues of $\boldsymbol{C}$ (order $10^{-1}$--$10^{0}$), the perturbation is negligible.
\end{itemize}
Both operations introduce small, controlled deviations from exact invariance while providing essential numerical stability and noise suppression. The ablation in Tab.~6 of the main paper confirms that these deviations are practically benign.
\end{remark}

\subsection{Proof of Isometric Vectorization}
\label{app:proof_isometry}

\begin{proposition}[Isometry of the Vectorization Map]
\label{prop:isometry}
The vectorization operator $\psi: \mathrm{Sym}(d) \to \mathbb{R}^{d(d+1)/2}$ defined with $\sqrt{2}$ scaling for off-diagonal entries is a linear isometry with respect to the Frobenius inner product:
\begin{equation}
    \langle \psi(\boldsymbol{A}), \psi(\boldsymbol{B}) \rangle = \langle \boldsymbol{A}, \boldsymbol{B} \rangle_F = \mathrm{tr}(\boldsymbol{A}\boldsymbol{B}), \quad \forall\, \boldsymbol{A}, \boldsymbol{B} \in \mathrm{Sym}(d).
\end{equation}
\end{proposition}
\begin{proof}
For symmetric matrices, the Frobenius inner product decomposes as:
\begin{equation}
    \langle \boldsymbol{A}, \boldsymbol{B} \rangle_F = \sum_i A_{ii}B_{ii} + 2\sum_{i < j} A_{ij}B_{ij}.
\end{equation}
The Euclidean inner product of the vectorized forms $\boldsymbol{v}_a = \psi(\boldsymbol{A})$ and $\boldsymbol{v}_b = \psi(\boldsymbol{B})$ is:
\begin{align}
    \langle \boldsymbol{v}_a, \boldsymbol{v}_b \rangle &= \sum_i A_{ii}B_{ii} + \sum_{i < j} (\sqrt{2}\,A_{ij})(\sqrt{2}\,B_{ij}) \\
    &= \sum_i A_{ii}B_{ii} + 2\sum_{i < j} A_{ij}B_{ij} = \langle \boldsymbol{A}, \boldsymbol{B} \rangle_F.
\end{align}
This establishes a Hilbert space isomorphism, ensuring that Euclidean retrieval on the vectorized descriptors is equivalent to computing distances on the manifold tangent space.
\end{proof}

\subsection{Newton--Schulz Convergence for SPD Matrices}
\label{app:ns_convergence}

\begin{proposition}[Convergence Guarantee under Frobenius Normalization]
\label{prop:ns_convergence}
Let $\boldsymbol{C} \in \mathcal{S}_{++}^d$ with eigenvalues $\lambda_1 \geq \cdots \geq \lambda_d > 0$. After Frobenius normalization $\boldsymbol{A}_0 = \boldsymbol{C} / \|\boldsymbol{C}\|_F$, all eigenvalues of $\boldsymbol{A}_0$ lie in $(0, 1]$, ensuring convergence of the coupled Newton--Schulz iteration.
\end{proposition}
\begin{proof}
For any SPD matrix, $\lambda_i > 0$ and $\|\boldsymbol{C}\|_F = \sqrt{\sum_{j=1}^d \lambda_j^2}$. Since all $\lambda_j > 0$:
\begin{equation}
    \frac{\lambda_i}{\|\boldsymbol{C}\|_F} = \frac{\lambda_i}{\sqrt{\sum_{j=1}^d \lambda_j^2}} \leq \frac{\lambda_i}{\sqrt{\lambda_i^2}} = 1,
\end{equation}
with equality iff $\boldsymbol{C} = \lambda_i \boldsymbol{I}$. Positivity is preserved since $\lambda_i > 0$ and $\|\boldsymbol{C}\|_F > 0$. Thus all eigenvalues of $\boldsymbol{A}_0$ lie in $(0, 1]$, satisfying the convergence condition of the coupled Newton--Schulz iteration.
\end{proof}

\subsection{NS Approximation Error Analysis}
\label{app:ns_error}

To quantify the quality of the Newton--Schulz approximation, we measure the relative Frobenius error $\|\boldsymbol{Y}_K - \boldsymbol{C}^{1/2}\|_F / \|\boldsymbol{C}^{1/2}\|_F$ where $\boldsymbol{C}^{1/2}$ is computed via exact eigendecomposition. We evaluate across all images in Pitts30k-test and Nardo-Air.

\begin{table}[ht]
    \centering
    \caption{\textbf{Newton--Schulz approximation error.} Relative Frobenius error $\|\boldsymbol{Y}_K - \boldsymbol{C}^{1/2}\|_F / \|\boldsymbol{C}^{1/2}\|_F$ (mean $\pm$ std) across all covariance matrices in each dataset, for different iteration counts $K$.}
    \label{tab:ns_error}
    \resizebox{\linewidth}{!}{
    \begin{tabular}{cccc}
        \toprule
        \textbf{NS Iterations $K$} & \textbf{Pitts30k-test} & \textbf{Nardo-Air} & \textbf{Hawkins} \\
        \midrule
        $K=1$ & 0.506717 $\pm$ 0.006906 & 0.507902 $\pm$ 0.004183 & 0.510287 $\pm$ 0.003775 \\
        $K=3$ & 0.280311 $\pm$ 0.012339 & 0.282664 $\pm$ 0.007795 & 0.294421 $\pm$ 0.004216 \\
        $K=5$ & 0.082801 $\pm$ 0.011338 & 0.079275 $\pm$ 0.007487 & 0.095820 $\pm$ 0.004164 \\
        $K=10$ & 0.000101 $\pm$ 0.000242 & 0.000576 $\pm$ 0.000833 & 0.000071 $\pm$ 0.000011 \\
        \bottomrule
    \end{tabular}
    }
\end{table}

Tab.~\ref{tab:ns_error} reports the relative Frobenius error $|\boldsymbol{Y}_K - \boldsymbol{C}^{1/2}|_F / |\boldsymbol{C}^{1/2}|_F$ for $K \in {1,3,5,10}$ on Pitts30k, Nardo-Air, and Hawkins. Error decays rapidly with $K$: from ${\sim}51\%$ at $K{=}1$ to ${\sim}9\%$ at $K{=}5$ and below $0.06\%$ at $K{=}10$. Per-image variance is consistently small across all three datasets, confirming that NS convergence is stable and largely data-agnostic regardless of scene type. We adopt $K{=}3$ as our default, striking a practical accuracy-efficiency tradeoff. The resulting ${\sim}28\%$ Frobenius error does not degrade retrieval performance because the downstream $\ell_2$-normalization step renders the descriptor insensitive to global scale, and the relative geometric structure of the covariance---which is what discriminates between places---is already faithfully preserved at this iteration count.

\section{Extended Ablation Studies}
\label{app:ablations}

\subsection{Random Projection Seed Sensitivity}
\label{app:seed_sensitivity}

To assess the stability of RIA with respect to the random orthogonal projection matrix $\boldsymbol{P}$, we repeat the full evaluation pipeline with 10 random seeds $\{0,1,10,42,101,1001,10001,10002,10010,100001\}$ while keeping all other hyperparameters fixed ($d{=}64$, $\tau{=}10^{-5}$, $\epsilon{=}10^{-4}$, $K{=}3$).

\begin{table}[ht]
    \centering
    \caption{\textbf{Random projection seed sensitivity.} R@1 across 10 random seeds for the projection matrix $\boldsymbol{P}$. All other hyperparameters are fixed. We report individual seed results and summary statistics.}
    \label{tab:seed_sensitivity}
    \resizebox{\linewidth}{!}{
    \begin{tabular}{lccccc}
        \toprule
        \textbf{Seed} & \textbf{17Places} & \textbf{Pitts30k} & \textbf{Gardens} & \textbf{Hawkins} & \textbf{Mid-Atlantic} \\
        \midrule
        0  & 64.78 & 86.53 & 96.00 & 50.85 & 27.72 \\
        1  & 63.30 & 86.56 & 96.00 & 50.85 & 25.74 \\
        10 & 63.79 & 86.41 & 97.00 & 45.76 & 31.68 \\
        42 & 63.55 & 86.36 & 97.00 & 52.54 & 31.68 \\
        101 & 65.02 & 85.87 & 96.00 & 50.00 & 37.62 \\
        1001 & 63.55 & 85.74 & 97.00 & 46.61 & 27.72 \\
        10001 & 63.79 & 84.79 & 97.50 & 52.54 & 30.69 \\
        10002 & 64.04 & 86.71 & 97.00 & 47.46 & 29.70 \\
        10010 & 64.29 & 85.09 & 97.00 & 48.31 & 28.71 \\
        100001 & 64.29 & 86.05 & 97.00 & 43.22 & 32.67 \\
        \midrule
        Mean $\pm$ Std & 64.04 $\pm$ 0.53 & 86.01 $\pm$ 0.61 & 96.75 $\pm$ 0.51 & 48.81 $\pm$ 2.92 & 30.40 $\pm$ 3.16 \\
        \bottomrule
    \end{tabular}
    }
\end{table}

Tab.~\ref{tab:seed_sensitivity} reports R@1 across 10 random seeds for the projection matrix $\boldsymbol{P} \in \mathrm{St}(D, d)$. On structured benchmarks, performance is remarkably stable: the standard deviation is below 0.61 R@1 points for all three datasets (17Places, Pitts30k, Gardens), and individual seeds rarely deviate more than 1 point from the mean. Unstructured scenes (Hawkins, Mid-Atlantic) exhibit moderately higher variance (std $\approx$ 3.0), reflecting the greater visual diversity and sparser coverage typical of non-urban environments. This seed sensitivity is in line with theoretical expectation: the congruence invariance of SPD descriptors ($\boldsymbol{C} = \boldsymbol{P}^\top \boldsymbol{\Sigma} \boldsymbol{P}$) guarantees that any orthonormal $\boldsymbol{P}$ preserves the Riemannian geometry up to a change of basis, so the choice of seed has no systematic effect on discriminability. The larger variance on unstructured datasets stems not from instability in the geometric framework but from the fact that different projection bases emphasize different feature dimensions, which matters more when scene appearance is heterogeneous. Overall, the results confirm that seed selection is a minor concern in practice; we fix seed $= 42$ throughout.

\subsection{PCA vs.\ Random Orthogonal Projection}
\label{app:pca_vs_random}

We compare three projection strategies under the same target dimension $d{=}64$: (a) random orthogonal projection (default), (b) PCA fitted on the database features of each dataset, and (c) PCA fitted on a generic ImageNet subset. This isolates whether RIA's performance stems from the SPD geometry or from a favorable projection basis.

\begin{table}[ht]
    \centering
    \caption{\textbf{Projection strategy comparison.} R@1 under different dimensionality reduction methods with $d{=}64$.}
    \label{tab:pca_vs_random}
    \resizebox{\linewidth}{!}{
    \begin{tabular}{lcccc}
        \toprule
        \textbf{Projection} & \textbf{Pitts30k} & \textbf{Nardo-Air} & \textbf{Hawkins} & \textbf{Laurel} \\
        \midrule
        Random Orthogonal (default) & 86.36 & 90.14 & 52.54 & 48.21 \\
        PCA (database-fitted)       & 84.76 & 81.69 & 43.22 & 47.32 \\
        PCA (ImageNet-fitted)       & 85.48 & 63.38 & 36.44 & 33.93 \\
        \bottomrule
    \end{tabular}
    }
\end{table}

Tab.~\ref{tab:pca_vs_random} shows that random orthogonal projection consistently outperforms both PCA variants across all datasets, with the gap widening substantially on unstructured scenes. On Pitts30k, the three strategies perform comparably (within 1.6 R@1 points), but on Nardo-Air and Hawkins, PCA deteriorates by up to 8--9 points for database-fitted PCA and over 26 points for ImageNet-fitted PCA. This result has a clean theoretical interpretation. PCA retains the directions of maximum variance in the feature space, but maximum variance is not equivalent to maximum discriminability for covariance-based descriptors: the off-diagonal covariance entries---which capture inter-feature correlations and carry much of the scene-discriminative signal---are disproportionately suppressed when projection is aligned with the principal components. Random orthogonal projection, by contrast, distributes information more evenly across the projected subspace and makes no assumptions about the feature distribution, which is precisely why it generalizes better to unstructured domains where DINOv2 feature statistics differ markedly from urban training distributions. The particularly severe degradation of ImageNet-fitted PCA on aerial and cave datasets confirms that PCA introduces an implicit domain bias that random projection avoids entirely.

\subsection{Extended $\tau$/$\epsilon$ Sensitivity Grid}
\label{app:tau_eps_grid}

We provide a comprehensive grid search over the ReCov threshold $\tau$ and the SPD regularization constant $\epsilon$, fixing $d{=}64$, $\alpha{=}0.5$, and $K{=}3$ (NS solver).
\begin{table}[ht]
    \centering
    \caption{\textbf{Pitts30k R@1 under varying $\tau$ and $\epsilon$.} Best result per column in bold.}
    \label{tab:tau_eps_pitts}
    \resizebox{\linewidth}{!}{
    \begin{tabular}{l cccccc}
        \toprule
        \diagbox{$\tau$}{$\epsilon$} & $0$ & $10^{-7}$ & $10^{-6}$ & $10^{-5}$ & $10^{-4}$ & $10^{-3}$ \\
        \midrule
        $0$         & 86.27 & 86.19 & \textbf{86.54} & 86.28 & 86.52 & \textbf{85.68} \\
        $10^{-7}$   & 86.19 & 86.31 & 86.24 & 86.63 & 86.58 & 85.31 \\
        $10^{-6}$   & 86.38 & 86.44 & 86.51 & 86.63 & 86.74 & 85.46 \\
        $10^{-5}$   & \textbf{86.47} & \textbf{86.49} & 86.52 & \textbf{86.65} & \textbf{86.81} & 85.52 \\
        $10^{-4}$   & 85.67 & 85.67 & 85.65 & 85.64 & 85.70 & 84.18 \\
        $10^{-3}$   & 68.23 & 68.25 & 68.27 & 68.35 & 68.52 & 66.77 \\
        \bottomrule
    \end{tabular}
    }
\end{table}

\begin{table}[ht]
    \centering
    \caption{\textbf{Nardo-Air R@1 under varying $\tau$ and $\epsilon$.} Best result per column in bold.}
    \label{tab:tau_eps_nardo}
    \resizebox{\linewidth}{!}{
    \begin{tabular}{l cccccc}
        \toprule
        \diagbox{$\tau$}{$\epsilon$} & $0$ & $10^{-7}$ & $10^{-6}$ & $10^{-5}$ & $10^{-4}$ & $10^{-3}$ \\
        \midrule
        $0$         & 90.11 & 90.04 & 90.38 & \textbf{91.08} & 91.05 & \textbf{90.22} \\
        $10^{-7}$   & 90.03 & 90.17 & 90.07 & 90.52 & 91.19 & 89.63 \\
        $10^{-6}$   & 90.24 & 90.21 & \textbf{90.61} & 90.59 & 91.33 & 89.84 \\
        $10^{-5}$   & \textbf{90.42} & \textbf{90.48} & 90.55 & 90.84 & \textbf{91.55} & 90.07 \\
        $10^{-4}$   & 60.56 & 60.56 & 60.56 & 61.97 & 69.01 & 59.15 \\
        $10^{-3}$   & 22.25 & 22.28 & 22.31 & 22.42 & 22.67 & 18.31 \\
        \bottomrule
    \end{tabular}
    }
\end{table}

Tabs.~\ref{tab:tau_eps_pitts}--\ref{tab:tau_eps_nardo} jointly reveal the sensitivity landscape of RIA to its two regularization hyperparameters. Several patterns are worth noting.

Tolerance to small $\tau$. Performance is largely flat for $\tau \leq 10^{-5}$ on both datasets: the top four rows in each table differ by at most 0.5 R@1 points within the same $\epsilon$ column. This plateau arises because very small thresholds retain nearly all covariance entries, and the slight improvement at $\tau{=}10^{-5}$ reflects the benefit of suppressing weak, noise-dominated off-diagonal entries that contribute little discriminative signal.

Cliff at large $\tau$. Beyond $\tau{=}10^{-4}$, performance collapses, and the collapse is far more severe on Nardo-Air (90$\rightarrow$61 at $\tau{=}10^{-4}$, 90$\rightarrow$22 at $\tau{=}10^{-3}$) than on Pitts30k (86$\rightarrow$86 and 86$\rightarrow$68, respectively). This asymmetry reflects a fundamental difference in covariance structure: unstructured aerial scenes encode scene identity through a richer set of inter-feature correlations, so aggressive thresholding removes discriminative off-diagonal structure that structured urban scenes can afford to lose.

Mild sensitivity to $\epsilon$. The $\epsilon$ dimension shows a gentler landscape: performance peaks around $\epsilon{=}10^{-4}$ and degrades slightly only at $\epsilon{=}10^{-3}$, where excessive additive regularization suppresses eigenvalue spread and reduces descriptor contrast. Based on this analysis, we recommend $\tau{=}10^{-5}$ and $\epsilon{=}10^{-4}$ as defaults, which sit at the sweet spot of the plateau and generalize robustly across both structured and unstructured domains.

\subsection{ReCov vs.\ Shrinkage Estimators}
\label{app:recov_vs_shrinkage}

We compare the hard thresholding operator (ReCov) against standard shrinkage estimators for covariance regularization. All variants use the same downstream pipeline ($\alpha{=}0.5$, NS $K{=}3$, $\epsilon{=}10^{-4}$).

\begin{table}[ht]
    \centering
    \caption{\textbf{Covariance rectification strategy comparison.} R@1 under different covariance estimation/regularization approaches.}
    \label{tab:recov_vs_shrinkage}
    \resizebox{\linewidth}{!}{
    \begin{tabular}{lcccc}
        \toprule
        \textbf{Rectification Method} & \textbf{Pitts30k} & \textbf{Nardo-Air} & \textbf{Hawkins} & \textbf{Laurel} \\
        \midrule
        None (raw covariance)               & 86.41 & 90.14 & 52.54 & 49.11 \\
        Hard Thresholding (ReCov, default)   & 86.36 & 90.14 & 52.54 & 48.21 \\
        Ledoit--Wolf Shrinkage               & 86.41 & 90.14 & 52.54 & 49.11 \\
        Oracle Approx.\ Shrinkage (OAS)     & 86.40 & 90.14 & 52.54 & 49.11 \\
        \bottomrule
    \end{tabular}
    }
\end{table}

Tab.~\ref{tab:recov_vs_shrinkage} shows that all four covariance rectification strategies perform within 1 R@1 point of each other across all datasets, with Nardo-Air and Hawkins returning identical scores regardless of the estimator. This near-equivalence has a straightforward explanation: the dominant SPD regularization in our pipeline is the additive ReEig step ($\epsilon{=}10^{-4}$, applied downstream), which already shifts all eigenvalues away from zero. Ledoit-Wolf and OAS both reduce to an additive scalar correction of the covariance matrix, so their effect is largely subsumed by the ReEig step; applying them on top yields no measurable benefit. Raw covariance without any prior rectification performs equally well for the same reason.

ReCov operates differently---it zeros out small off-diagonal entries before the ReEig step, imposing sparse structure rather than scalar regularization. The tiny gap on Pitts30k and Laurel (at most 0.9 points) indicates that the fixed threshold $\tau{=}10^{-5}$ discards a small amount of genuine signal on these datasets. Nevertheless, ReCov remains our preferred default for two reasons: it avoids the matrix inverse required by shrinkage estimators (keeping the full pipeline at $O(d^2)$ per descriptor), and it provides an explicit structural prior---covariance sparsity---that is theoretically motivated by the locality of patch-feature dependencies. Taken together, the results confirm that the primary driver of RIA's discriminability is the Riemannian geometric mapping (matrix square root and vectorization), not the specific choice of upstream rectification.

\subsection{DINOv2-RIA-FT Component-wise Ablation}
\label{app:ft_ablation}

We progressively add each fine-tuning component to isolate its contribution. All variants are trained on GSV-Cities with identical optimization settings.

\begin{table}[ht]
    \centering
    \caption{\textbf{Component-wise ablation of DINOv2-RIA-FT.} Starting from training-free RIA, we incrementally add backbone unfreezing, learnable projection, multi-head design, and per-head MLP compression. R@1 is reported.}
    \label{tab:ft_ablation}
    \resizebox{\linewidth}{!}{
    \begin{tabular}{lccccc}\toprule
        \textbf{Configuration} & \textbf{Learnable Proj.} & \textbf{Multi-head} & \textbf{MLP} & \textbf{Unfreeze BB} & \textbf{Pitts30k} \\
        \midrule
        Training-free RIA        & \texttimes & \texttimes & \texttimes & \texttimes & 86.7 \\
        + Backbone Unfreeze      & \texttimes & \texttimes & \texttimes & \checkmark & 89.3 \\
        + Learnable Projection   & \checkmark & \texttimes & \texttimes & \checkmark & 90.8 \\
        + Multi-head ($H{=}4$)   & \checkmark & \checkmark & \texttimes & \checkmark & 92.0 \\
        \textbf{+ Per-head MLP (full)}  & \checkmark & \checkmark & \checkmark & \checkmark & \textbf{93.5} \\
        \bottomrule
    \end{tabular}
    }
\end{table}

Tab.~\ref{tab:ft_ablation} traces the incremental contribution of each fine-tuning component on Pitts30k. Unfreezing the backbone yields the largest single gain (+2.6 pts), confirming that backbone feature quality is the primary performance bottleneck in the training-free regime. The learnable projection (+1.5 pts), multi-head design (+1.2 pts), and per-head MLP (+1.5 pts) each contribute consistent further improvements, bringing the total supervised gain to +6.8 pts over the training-free baseline. The consistent additive increments suggest that each component addresses a distinct aspect of the representation: the learnable projection adapts the subspace selection to the training distribution, the multi-head design encourages specialization across different covariance modes, and the per-head MLP provides non-linear compression that further concentrates discriminative information.

\section{Implementation and Reproducibility}
\label{app:implementation}

\subsection{DINOv2-RIA-FT Training Recipe}
\label{app:training_recipe}

The supervised DINOv2-RIA-FT variant builds on a DINOv2-Large backbone with partial backbone adaptation and a lightweight multi-head RIA aggregation module. Training is performed on GSV-Cities with standard metric learning objectives and data augmentation, while preserving the core geometric design of the training-free descriptor: low-dimensional projection, covariance pooling, SPD rectification, and power-based manifold mapping. We intentionally keep the appendix focused on the high-level design choices needed to interpret the ablations, without disclosing a fully exhaustive recipe.

\subsection{Backbone Swap Setup (Tab.~3 in Main Paper)}
\label{app:backbone_swap}

For the backbone generality experiment (Tab.~3 in the main paper), we use each compared method's \textbf{backbone encoder only} to extract local patch features, discarding their native aggregation head. Specifically:
\begin{itemize}
    \item \textbf{MixVPR + RIA}: We use the ResNet50-based backbone from MixVPR to extract $N$ local feature maps, reshape them into $N \times D_{\text{in}}$ patch descriptors, and apply the RIA pipeline (random projection $\to$ covariance $\to$ ReCov $\to$ PEM $\to$ vectorization).
    \item \textbf{SALAD + RIA}: We use the DINOv2-ViT-B/14 backbone from SALAD to extract patch tokens, then apply RIA.
    \item \textbf{SelaVPR + RIA}: We use the DINOv2-ViT-B/14 backbone with SelaVPR's trained adapters to extract adapted patch features, then apply RIA.
\end{itemize}
In all cases, the RIA hyperparameters are identical to the default DINOv2-RIA configuration ($d{=}64$, $\tau{=}10^{-5}$, $\epsilon{=}10^{-4}$, $K{=}3$). No retraining or tuning is performed.

\section{Additional Qualitative Analysis}
\label{app:qualitative}

\subsection{Extended Perturbation Study: Homography-based Viewpoint Change}
\label{app:homography}

To go beyond the planar rotation and brightness perturbations in the main paper (Fig.~3), we apply random homography transformations that simulate projective viewpoint changes. We visualize two complementary comparisons: a purely structured pair (17Places and Gardens) and a mixed structured/unstructured pair (Pitts30k and Hawkins). For each test image, we sample random homographies with increasing magnitude and measure the feature distance drift ($1 - \text{cosine\_similarity}$) of RIA, MixVPR, and SALAD descriptors.

\begin{figure*}[t]
    \centering
    \begin{subfigure}[t]{0.49\textwidth}
        \centering
        \includegraphics[width=\linewidth]{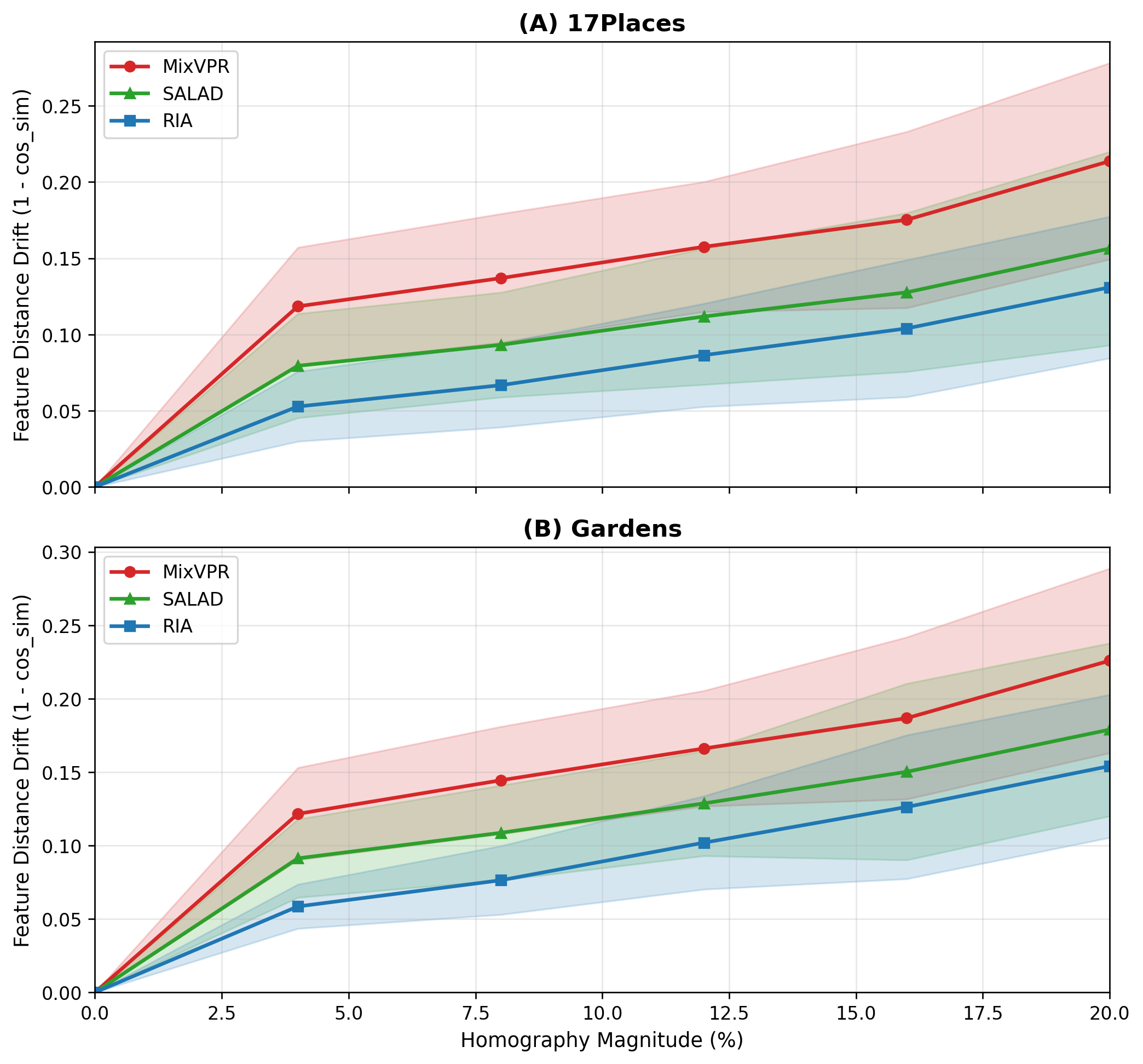}
        \caption{17Places and Gardens.}
    \end{subfigure}\hfill
    \begin{subfigure}[t]{0.49\textwidth}
        \centering
        \includegraphics[width=\linewidth]{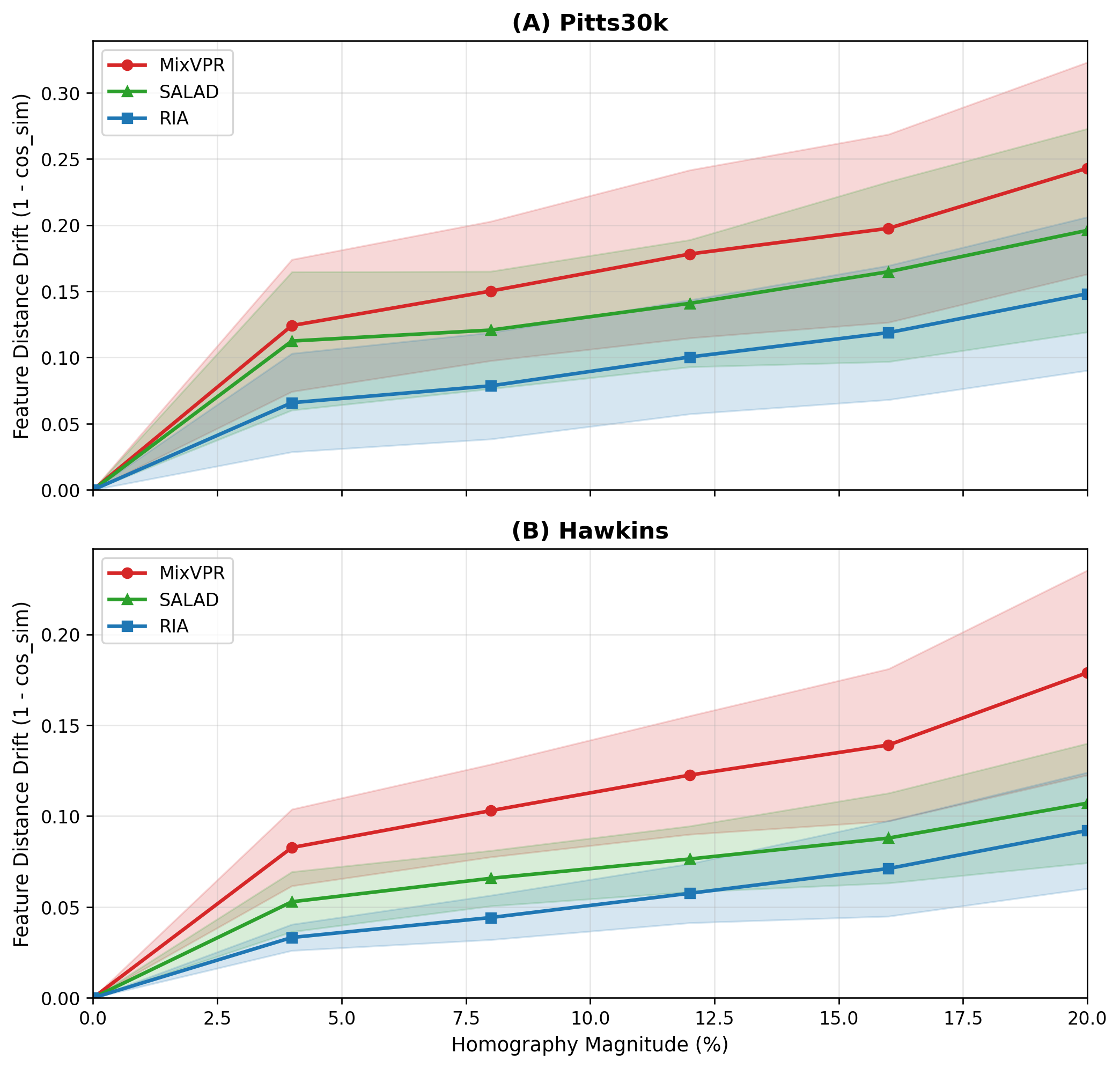}
        \caption{Pitts30k and Hawkins.}
    \end{subfigure}
    \caption{\textbf{Feature distance drift under random homography perturbations.} We apply projective transformations of increasing magnitude and compare descriptor drift for MixVPR, SALAD, and RIA. Across both the structured pair (17Places, Gardens) and the cross-regime pair (Pitts30k, Hawkins), RIA consistently exhibits the lowest drift curve over the perturbation range.}
    \label{fig:homography_drift}
\end{figure*}

Fig.~\ref{fig:homography_drift} plots feature distance drift ($1 - \cos\text{sim}$) as a function of homography magnitude across four datasets spanning both structured and unstructured settings. RIA maintains the lowest drift curve in all four panels, with a consistent margin over both MixVPR and SALAD throughout the full perturbation range.

On the structured benchmarks (17Places, Gardens, Pitts30k), the ordering is clear and stable: MixVPR exhibits the highest drift, SALAD is an intermediate, and RIA is substantially more stable---reaching approximately half the drift of MixVPR at magnitude 20\%. This is consistent with the theoretical properties of covariance descriptors: second-order statistics aggregate spatial information in a way that is naturally more stable under projective transformations than the spatial attention or global pooling mechanisms used by SALAD and MixVPR, respectively. The narrow confidence bands for RIA on these datasets further indicate that its geometric stability is consistent across images, rather than being an artifact of a few easy samples.

On Hawkins (an unstructured underground corridor), the absolute drift values for all three methods are lower, and SALAD narrows the gap with RIA at large magnitudes. This suggests that the unstructured appearance of cave scenes affords less spatial anchoring for all methods, making drift metrics converge. Nevertheless, RIA retains a clear advantage at small-to-moderate perturbation magnitudes ($\leq 10\%$), where practical viewpoint changes are most likely to occur. Across all four datasets, RIA never yields the highest drift in any panel, confirming that the Riemannian geometric framework provides a robust and general form of viewpoint invariance beyond the planar augmentations studied in the main paper.

\end{document}